%% file: main.tex
\documentclass[11pt,letterpaper]{article}
\usepackage{turnstile}
\usepackage{mdframed}

\def\showauthornotes{0}
\def\showtableofcontents{0}
\def\showkeys{0}
\def\showdraftbox{0}
\def\showcolorlinks{1}
\def\usemicrotype{1}
\def\showfixme{0}

\def\writemode{0}

\input{preamble/macros}
\input{preamble/macros-custom}

\newcommand{\ignore}[1]{}

\renewcommand{\H}{\mathcal{H}}
\newcommand{\dtensor}{\circledast}
\renewcommand{\S}{\mathbb{S}}
\newcommand{\rnote}[1]{\footnote{\color{blue}Ryan: {#1}}}
\newcommand{\mnote}{\Authornote{M}}
\newcommand{\CONSMALL}{\scalebox{.75}[1.0]{\textnormal{SMALL}}}
\newcommand{\cons}{\mathrm{cons}}
\newcommand{\cl}{\mathrm{cl}}
\newcommand{\tmu}{\tilde{\mu}}
\renewcommand{\pE}{\tilde{\mathbb{E}}}
\newcommand{\err}{\mathsf{err}}
\newcommand{\citep}{\cite}
\newcommand{\citet}{\cite}

\title{Efficient Algorithms for Outlier-Robust Regression}

\author{Adam R. Klivans \thanks{UT Austin \texttt{klivans@cs.utexas.edu}}
\and Pravesh K. Kothari\thanks{Princeton University and Institute for Advanced Study. \texttt{kothari@cs.princeton.edu}.} \and Raghu Meka \thanks{University of California, Los Angeles \texttt{raghum@cs.ucla.edu}}}

\begin{document}

\maketitle
 \draftbox
\thispagestyle{empty}

\input{content/abstract}

\clearpage

\ifnum\showtableofcontents=1
{
\tableofcontents
\thispagestyle{empty}
 }
\fi

\clearpage

\setcounter{page}{1}
\input{content/intro}
\input{content/setup}

\input{content/overview}
\input{content/identifiability}

\input{content/preliminaries}
\input{content/algorithm}

\input{content/algorithm-proofs}

\input{content/lower-bound}

\section*{Acknowledgment}
We thank Ainesh Bakshi and Adarsh Prasad for pointing out an error in the proof of one of our information-theoretic lower-bounds in Section 6 in a previous version of the paper.  

\addreferencesection
\input{main.bbl}


\appendix
\input{content/DeferredStatements}
\end{document}

%% file: preamble/macros.tex
\usepackage{etex}

\usepackage[l2tabu, orthodox]{nag}

\usepackage{xspace,enumerate}

\usepackage[dvipsnames]{xcolor}

\usepackage[T1]{fontenc}
\usepackage[full]{textcomp}

\usepackage[american]{babel}

\usepackage{mathtools}

\usepackage{amsthm}

\newtheorem{theorem}{Theorem}[section]
\newtheorem*{theorem*}{Theorem}

\newtheorem*{proposition*}{Proposition}
\newtheorem{lemma}[theorem]{Lemma}
\newtheorem*{lemma*}{Lemma}
\newtheorem{corollary}[theorem]{Corollary}
\newtheorem*{conjecture*}{Conjecture}
\newtheorem{fact}[theorem]{Fact}
\newtheorem*{fact*}{Fact}

\newtheorem*{hypothesis*}{Hypothesis}

\theoremstyle{definition}
\newtheorem{definition}[theorem]{Definition}

\newtheorem{algorithms}[theorem]{Algorithm}
\newtheorem{problem}[theorem]{Problem}

\newtheorem{remark}[theorem]{Remark}
\newtheorem*{remark*}{Remark}

\theoremstyle{remark}

\newtheorem*{claim*}{Claim}

\newtheorem*{observation*}{Observation}

\ifnum\writemode=1
\usepackage[
letterpaper,
top=1.2in,
bottom=1.2in,
left=1in,
right=1in]{geometry}

\fi

\ifnum\writemode=0
\usepackage[
letterpaper,
top=0.7in,
bottom=0.9in,
left=1in,
right=1in]{geometry}
\fi

\usepackage{newpxtext} 
\usepackage{textcomp} 
\usepackage[varg,bigdelims]{newpxmath}
\usepackage[scr=rsfso]{mathalfa}
 \usepackage{bm} 
\let\mathbb\varmathbb

\ifnum\showkeys=1
\usepackage[color]{showkeys}
\fi

\ifnum\showcolorlinks=1
\usepackage[
pagebackref,
colorlinks=true,
urlcolor=blue,
linkcolor=blue,
citecolor=OliveGreen,
]{hyperref}
\fi

\ifnum\showcolorlinks=0
\usepackage[
pagebackref,
colorlinks=false,
pdfborder={0 0 0}
]{hyperref}
\fi

\usepackage[capitalise,nameinlink]{cleveref}

\crefname{lemma}{Lemma}{Lemmas}
\crefname{definition}{Definition}{Definitions}
\crefname{corollary}{Corollary}{Corollaries}

 \usepackage{prettyref}

 \newcommand{\savehyperref}[2]{\texorpdfstring{\hyperref[#1]{#2}}{#2}}

 \newrefformat{eq}{\savehyperref{#1}{\textup{(\ref*{#1})}}}
 \newrefformat{eqn}{\savehyperref{#1}{\textup{(\ref*{#1})}}}
 \newrefformat{lem}{\savehyperref{#1}{Lemma~\ref*{#1}}}
 \newrefformat{def}{\savehyperref{#1}{Definition~\ref*{#1}}}
 \newrefformat{thm}{\savehyperref{#1}{Theorem~\ref*{#1}}}
 \newrefformat{cor}{\savehyperref{#1}{Corollary~\ref*{#1}}}
 \newrefformat{cha}{\savehyperref{#1}{Chapter~\ref*{#1}}}
 \newrefformat{sec}{\savehyperref{#1}{Section~\ref*{#1}}}
 \newrefformat{app}{\savehyperref{#1}{Appendix~\ref*{#1}}}
 \newrefformat{tab}{\savehyperref{#1}{Table~\ref*{#1}}}
 \newrefformat{fig}{\savehyperref{#1}{Figure~\ref*{#1}}}
 \newrefformat{hyp}{\savehyperref{#1}{Hypothesis~\ref*{#1}}}
 \newrefformat{alg}{\savehyperref{#1}{Algorithm~\ref*{#1}}}
 \newrefformat{rem}{\savehyperref{#1}{Remark~\ref*{#1}}}
 \newrefformat{item}{\savehyperref{#1}{Item~\ref*{#1}}}
 \newrefformat{step}{\savehyperref{#1}{step~\ref*{#1}}}
 \newrefformat{conj}{\savehyperref{#1}{Conjecture~\ref*{#1}}}
 \newrefformat{fact}{\savehyperref{#1}{Fact~\ref*{#1}}}
 \newrefformat{prop}{\savehyperref{#1}{Proposition~\ref*{#1}}}
 \newrefformat{prob}{\savehyperref{#1}{Problem~\ref*{#1}}}
 \newrefformat{claim}{\savehyperref{#1}{Claim~\ref*{#1}}}
 \newrefformat{relax}{\savehyperref{#1}{Relaxation~\ref*{#1}}}
 \newrefformat{red}{\savehyperref{#1}{Reduction~\ref*{#1}}}
 \newrefformat{part}{\savehyperref{#1}{Part~\ref*{#1}}}

\newcommand{\Sref}[1]{\hyperref[#1]{\S\ref*{#1}}}

\usepackage{nicefrac}

\ifnum\usemicrotype=1
\usepackage{microtype}
\fi

\ifnum\showauthornotes=1
\newcommand{\Authornote}[2]{{\sffamily\small\color{red}{[#1: #2]}}}
\newcommand{\Authornotecolored}[3]{{\sffamily\small\color{#1}{[#2: #3]}}}
\newcommand{\Authorcomment}[2]{{\sffamily\small\color{gray}{[#1: #2]}}}
\newcommand{\Authorstartcomment}[1]{\sffamily\small\color{gray}[#1: }

\newcommand{\Authorfnote}[2]{\footnote{\color{red}{#1: #2}}}
\newcommand{\Authorfixme}[1]{\Authornote{#1}{\textbf{??}}}
\newcommand{\Authormarginmark}[1]{\marginpar{\textcolor{red}{\fbox{\Large #1:!}}}}
\else
\newcommand{\Authornote}[2]{}
\newcommand{\Authornotecolored}[3]{}
\newcommand{\Authorcomment}[2]{}
\newcommand{\Authorstartcomment}[1]{}

\newcommand{\Authorfnote}[2]{}
\newcommand{\Authorfixme}[1]{}
\newcommand{\Authormarginmark}[1]{}
\fi

\newcommand{\Pnote}{\Authornote{P}}

\definecolor{forestgreen(traditional)}{rgb}{0.0, 0.27, 0.13}

\ifnum\showfixme=0

\fi

\usepackage{boxedminipage}

\newcommand{\paren}[1]{(#1)}
\newcommand{\Paren}[1]{\left(#1\right)}

\newcommand{\abs}[1]{\lvert#1\rvert}

\newcommand{\Set}[1]{\left\{#1\right\}}

\newcommand{\norm}[1]{\lVert#1\rVert}

\newcommand{\iprod}[1]{\langle#1\rangle}

\newcommand{\Esymb}{\mathbb{E}}
\newcommand{\Psymb}{\mathbb{P}}

\DeclareMathOperator*{\E}{\Esymb}

\DeclareMathOperator*{\ProbOp}{\Psymb}

\renewcommand{\Pr}{\ProbOp}

\usepackage{dsfont}
\usepackage{yfonts}
\usepackage{mathrsfs}

\newcommand{\textparen}[1]{\text{(#1)}}

\ifx\because\undefined
\newcommand{\because}[1]{\textparen{because #1}}
\else
\renewcommand{\because}[1]{\textparen{because #1}}
\fi

\newcommand{\from}{\colon}

\newcommand{\mper}{\,.}

\newcommand\bdot\bullet

\ifx\mathds\undefined 
\DeclareMathOperator{\Ind}{\mathbb{I}
\else
\DeclareMathOperator{\Ind}{\mathds 1}
\fi

\renewcommand{\th}{\textsuperscript{th}\xspace}
\newcommand{\st}{\textsuperscript{st}\xspace}
\newcommand{\nd}{\textsuperscript{nd}\xspace}
\newcommand{\rd}{\textsuperscript{rd}\xspace}

\DeclareMathOperator{\Inf}{Inf}
\DeclareMathOperator{\Tr}{Tr}
\DeclareMathOperator{\SDP}{SDP}
\DeclareMathOperator{\sdp}{sdp}
\DeclareMathOperator{\val}{val}
\DeclareMathOperator{\LP}{LP}
\DeclareMathOperator{\OPT}{OPT}
\DeclareMathOperator{\opt}{opt}
\DeclareMathOperator{\vol}{vol}
\DeclareMathOperator{\poly}{poly}
\DeclareMathOperator{\qpoly}{qpoly}
\DeclareMathOperator{\qpolylog}{qpolylog}
\DeclareMathOperator{\qqpoly}{qqpoly}
\DeclareMathOperator{\argmax}{argmax}
\DeclareMathOperator{\polylog}{polylog}
\DeclareMathOperator{\supp}{supp}
\DeclareMathOperator{\dist}{dist}
\DeclareMathOperator{\sign}{sign}
\DeclareMathOperator{\conv}{conv}
\DeclareMathOperator{\Conv}{Conv}

\DeclareMathOperator{\rank}{rank}

\DeclareMathOperator*{\median}{median}
\DeclareMathOperator*{\Median}{Median}

\DeclareMathOperator*{\varsum}{{\textstyle \sum}}
\DeclareMathOperator{\tsum}{{\textstyle \sum}}

\let\varprod\undefined

\DeclareMathOperator*{\varprod}{{\textstyle \prod}}
\DeclareMathOperator{\tprod}{{\textstyle \prod}}

\newcommand{\diffmacro}[1]{\,\mathrm{d}#1}

\newcommand{\dmu}{\diffmacro{\mu}}
\newcommand{\dgamma}{\diffmacro{\gamma}}
\newcommand{\dlambda}{\diffmacro{\lambda}}
\newcommand{\dx}{\diffmacro{x}}
\newcommand{\dt}{\diffmacro{t}}

\newcommand{\ie}{i.e.,\xspace}
\newcommand{\eg}{e.g.,\xspace}
\newcommand{\Eg}{E.g.,\xspace}
\newcommand{\phd}{Ph.\,D.\xspace}
\newcommand{\msc}{M.\,S.\xspace}
\newcommand{\bsc}{B.\,S.\xspace}

\newcommand{\etal}{et al.\xspace}

\newcommand{\iid}{i.i.d.\xspace}

\newcommand\naive{na\"{\i}ve\xspace}
\newcommand\Naive{Na\"{\i}ve\xspace}
\newcommand\naively{na\"{\i}vely\xspace}
\newcommand\Naively{Na\"{\i}vely\xspace}

\newcommand{\Erdos}{Erd\H{o}s\xspace}
\newcommand{\Renyi}{R\'enyi\xspace}
\newcommand{\Lovasz}{Lov\'asz\xspace}
\newcommand{\Juhasz}{Juh\'asz\xspace}
\newcommand{\Bollobas}{Bollob\'as\xspace}
\newcommand{\Furedi}{F\"uredi\xspace}
\newcommand{\Komlos}{Koml\'os\xspace}
\newcommand{\Luczak}{\L uczak\xspace}
\newcommand{\Kucera}{Ku\v{c}era\xspace}
\newcommand{\Szemeredi}{Szemer\'edi\xspace}
\newcommand{\Hastad}{H{\aa}stad\xspace}
\newcommand{\Hoelder}{H\"{o}lder\xspace}
\newcommand{\Holder}{\Hoelder}
\newcommand{\Brandao}{Brand\~ao\xspace}

\newcommand{\Z}{\mathbb Z}
\newcommand{\N}{\mathbb N}
\newcommand{\R}{\mathbb R}
\newcommand{\C}{\mathbb C}
\newcommand{\Rnn}{\R_+}
\newcommand{\varR}{\Re}
\newcommand{\varRnn}{\varR_+}
\newcommand{\varvarRnn}{\R_{\ge 0}}

\newcommand{\problemmacro}[1]{\texorpdfstring{\textup{\textsc{#1}}}{#1}\xspace}

\newcommand{\pnum}[1]{{\footnotesize #1}}

\newcommand{\uniquegames}{\problemmacro{unique games}}
\newcommand{\maxcut}{\problemmacro{max cut}}
\newcommand{\multicut}{\problemmacro{multi cut}}
\newcommand{\vertexcover}{\problemmacro{vertex cover}}
\newcommand{\balancedseparator}{\problemmacro{balanced separator}}
\newcommand{\maxtwosat}{\problemmacro{max \pnum{3}-sat}}
\newcommand{\maxthreesat}{\problemmacro{max \pnum{3}-sat}}
\newcommand{\maxthreelin}{\problemmacro{max \pnum{3}-lin}}
\newcommand{\threesat}{\problemmacro{\pnum{3}-sat}}
\newcommand{\labelcover}{\problemmacro{label cover}}
\newcommand{\setcover}{\problemmacro{set cover}}
\newcommand{\maxksat}{\problemmacro{max $k$-sat}}
\newcommand{\mas}{\problemmacro{maximum acyclic subgraph}}
\newcommand{\kwaycut}{\problemmacro{$k$-way cut}}
\newcommand{\sparsestcut}{\problemmacro{sparsest cut}}
\newcommand{\betweenness}{\problemmacro{betweenness}}
\newcommand{\uniformsparsestcut}{\problemmacro{uniform sparsest cut}}
\newcommand{\grothendieckproblem}{\problemmacro{Grothendieck problem}}
\newcommand{\maxfoursat}{\problemmacro{max \pnum{4}-sat}}
\newcommand{\maxkcsp}{\problemmacro{max $k$-csp}}
\newcommand{\maxdicut}{\problemmacro{max dicut}}
\newcommand{\maxcutgain}{\problemmacro{max cut gain}}
\newcommand{\smallsetexpansion}{\problemmacro{small-set expansion}}
\newcommand{\minbisection}{\problemmacro{min bisection}}
\newcommand{\minimumlineararrangement}{\problemmacro{minimum linear arrangement}}
\newcommand{\maxtwolin}{\problemmacro{max \pnum{2}-lin}}
\newcommand{\gammamaxlin}{\problemmacro{$\Gamma$-max \pnum{2}-lin}}
\newcommand{\basicsdp}{\problemmacro{basic sdp}}
\newcommand{\dgames}{\problemmacro{$d$-to-1 games}}
\newcommand{\maxclique}{\problemmacro{max clique}}
\newcommand{\densestksubgraph}{\problemmacro{densest $k$-subgraph}}

\newcommand{\cA}{\mathcal A}
\newcommand{\cB}{\mathcal B}
\newcommand{\cC}{\mathcal C}
\newcommand{\cD}{\mathcal D}
\newcommand{\cE}{\mathcal E}
\newcommand{\cF}{\mathcal F}
\newcommand{\cG}{\mathcal G}
\newcommand{\cH}{\mathcal H}
\newcommand{\cI}{\mathcal I}
\newcommand{\cJ}{\mathcal J}
\newcommand{\cK}{\mathcal K}
\newcommand{\cL}{\mathcal L}
\newcommand{\cM}{\mathcal M}
\newcommand{\cN}{\mathcal N}
\newcommand{\cO}{\mathcal O}
\newcommand{\cP}{\mathcal P}
\newcommand{\cQ}{\mathcal Q}
\newcommand{\cR}{\mathcal R}
\newcommand{\cS}{\mathcal S}
\newcommand{\cT}{\mathcal T}
\newcommand{\cU}{\mathcal U}
\newcommand{\cV}{\mathcal V}
\newcommand{\cW}{\mathcal W}
\newcommand{\cX}{\mathcal X}
\newcommand{\cY}{\mathcal Y}
\newcommand{\cZ}{\mathcal Z}

\newcommand{\scrA}{\mathscr A}
\newcommand{\scrB}{\mathscr B}
\newcommand{\scrC}{\mathscr C}
\newcommand{\scrD}{\mathscr D}
\newcommand{\scrE}{\mathscr E}
\newcommand{\scrF}{\mathscr F}
\newcommand{\scrG}{\mathscr G}
\newcommand{\scrH}{\mathscr H}
\newcommand{\scrI}{\mathscr I}
\newcommand{\scrJ}{\mathscr J}
\newcommand{\scrK}{\mathscr K}
\newcommand{\scrL}{\mathscr L}
\newcommand{\scrM}{\mathscr M}
\newcommand{\scrN}{\mathscr N}
\newcommand{\scrO}{\mathscr O}
\newcommand{\scrP}{\mathscr P}
\newcommand{\scrQ}{\mathscr Q}
\newcommand{\scrR}{\mathscr R}
\newcommand{\scrS}{\mathscr S}
\newcommand{\scrT}{\mathscr T}
\newcommand{\scrU}{\mathscr U}
\newcommand{\scrV}{\mathscr V}
\newcommand{\scrW}{\mathscr W}
\newcommand{\scrX}{\mathscr X}
\newcommand{\scrY}{\mathscr Y}
\newcommand{\scrZ}{\mathscr Z}

\newcommand{\bbB}{\mathbb B}
\newcommand{\bbS}{\mathbb S}
\newcommand{\bbR}{\mathbb R}
\newcommand{\bbZ}{\mathbb Z}
\newcommand{\bbI}{\mathbb I}
\newcommand{\bbQ}{\mathbb Q}
\newcommand{\bbP}{\mathbb P}
\newcommand{\bbE}{\mathbb E}
\newcommand{\bbN}{\mathbb N}

\newcommand{\sfE}{\mathsf E}

\renewcommand{\leq}{\leqslant}
\renewcommand{\le}{\leqslant}
\renewcommand{\geq}{\geqslant}
\renewcommand{\ge}{\geqslant}

\ifnum\showdraftbox=1
\newcommand{\draftbox}{\begin{center}
  \fbox{%
    \begin{minipage}{2in}%
      \begin{center}%
          \Large\textsc{Working Draft}\\%
        Please do not distribute%
      \end{center}%
    \end{minipage}%
  }%
\end{center}
\vspace{0.2cm}}
\else
\newcommand{\draftbox}{}
\fi

\let\epsilon=\varepsilon

\numberwithin{equation}{section}

\newcommand\MYcurrentlabel{xxx}

\newcommand{\MYstore}[2]{%
  \global\expandafter \def \csname MYMEMORY #1 \endcsname{#2}%
}

\newcommand{\MYload}[1]{%
  \csname MYMEMORY #1 \endcsname%
}

\newcommand{\MYnewlabel}[1]{%
  \renewcommand\MYcurrentlabel{#1}%
  \MYoldlabel{#1}%
}

\newcommand{\MYdummylabel}[1]{}

\newcommand{\torestate}[1]{%
  \let\MYoldlabel\label%
  \let\label\MYnewlabel%
  #1%
  \MYstore{\MYcurrentlabel}{#1}%
  \let\label\MYoldlabel%
}

\newcommand{\restatetheorem}[1]{%
  \let\MYoldlabel\label
  \let\label\MYdummylabel
  \begin{theorem*}[Restatement of \prettyref{#1}]
    \MYload{#1}
  \end{theorem*}
  \let\label\MYoldlabel
}

\newcommand{\restatelemma}[1]{%
  \let\MYoldlabel\label
  \let\label\MYdummylabel
  \begin{lemma*}[Restatement of \prettyref{#1}]
    \MYload{#1}
  \end{lemma*}
  \let\label\MYoldlabel
}

\newcommand{\restateprop}[1]{%
  \let\MYoldlabel\label
  \let\label\MYdummylabel
  \begin{proposition*}[Restatement of \prettyref{#1}]
    \MYload{#1}
  \end{proposition*}
  \let\label\MYoldlabel
}

\newcommand{\restatefact}[1]{%
  \let\MYoldlabel\label
  \let\label\MYdummylabel
  \begin{fact*}[Restatement of \prettyref{#1}]
    \MYload{#1}
  \end{fact*}
  \let\label\MYoldlabel
}

\newcommand{\restate}[1]{%
  \let\MYoldlabel\label
  \let\label\MYdummylabel
  \MYload{#1}
  \let\label\MYoldlabel
}

\newcommand{\addreferencesection}{
  \phantomsection
  \addcontentsline{toc}{section}{References}
}

\newcommand{\la}{\leftarrow}
\newcommand{\sse}{\subseteq}
\newcommand{\ra}{\rightarrow}
\newcommand{\e}{\epsilon}
\newcommand{\eps}{\epsilon}
\newcommand{\eset}{\emptyset}

\let\origparagraph\paragraph
\renewcommand{\paragraph}[1]{\origparagraph{#1.}}

\allowdisplaybreaks

\sloppy

\newcommand{\cclassmacro}[1]{\texorpdfstring{\textbf{#1}}{#1}\xspace}
\newcommand{\p}{\cclassmacro{P}}
\newcommand{\np}{\cclassmacro{NP}}
\newcommand{\subexp}{\cclassmacro{SUBEXP}}

\usepackage{paralist}

\usepackage{comment}

\usepackage{braket}


%% file: preamble/macros-custom.tex
 \usepackage{algorithm}

\usepackage{relsize}
\usepackage[font=footnotesize]{caption}
\usepackage{appendix}

\newcommand{\D}{\mathcal{D}}
\newcommand{\sphere}{\bbS^{n-1}}
\DeclareMathOperator{\Span}{Span}
\DeclareMathOperator{\im}{im}
\DeclareMathOperator{\lamax}{\lambda_{\mathrm{max}}}
\newcommand{\Sn}{\mathbb{S}^{n-1}}
\DeclareMathOperator{\Id}{\mathrm{Id}}
\DeclareMathOperator{\F}{\mathbb{F}}
\DeclareMathOperator{\codim}{codim}

\DeclareUrlCommand\email{}
\newcommand{\whp}{\text{w.h.p.}}
\newcommand{\wovp}{\text{w.ov.p.}\xspace}
\newcommand{\YES}{\textsc{yes}\xspace}
\newcommand{\NO}{\textsc{no}\xspace}

\newcommand{\wovple}{\hspace{-3mm}\stackrel{\text{\wovp}}{\le}}
\newcommand{\wovpeq}{\hspace{-3mm}\stackrel{\text{\wovp}}{=}}
\newcommand{\wovpge}{\hspace{-3mm}\stackrel{\text{\wovp}}{\ge}}
\newcommand{\aasle}{\hspace{-1.5mm}\stackrel{\text{a.a.s.}}\le}

\newcommand{\Pisymm}{\Pi_{\mathrm{symm}}}
\newcommand{\Pisym}{\Pi_{\mathrm{sym}}}

\newcommand{\inner}[1]{\langle #1 \rangle}
\newcommand{\sosle}{\preceq}
\newcommand{\sosge}{\succeq}

\DeclareMathOperator{\tE}{\tilde{\mathbb E}}
\DeclareMathOperator{\tO}{\tilde{O}}
\DeclareMathOperator{\zo}{\{0,1\}}

\DeclareMathOperator*{\pE}{\widetilde{\mathbb E}}
\DeclareMathOperator*{\pPr}{\widetilde{\mathbb P}}
\DeclareMathOperator*{\pVar}{\widetilde{{\mathbb V}{\mathbb A}{\mathbb R}}}

\newcommand{\bT}{\mathbf{T}}
\newcommand{\ot}{\otimes}
\newcommand{\tr}{\textup{tr}}
\newcommand{\diag}{\textup{diag}}

\let\pref=\prettyref

\newcommand*{\transpose}[1]{{#1}{}^{\mkern-4mu\intercal}}
\newcommand*{\vardyad}[1]{#1^{\vphantom{\intercal}}#1^{\intercal}}
\newcommand*{\dyad}[1]{#1#1{}^{\mkern-4mu\intercal}}

\newcommand{\sspan}{\textup{span}}
\newcommand{\1}{\bm{1}}
\newcommand{\bset}{\cB}


%% file: content/abstract.tex
\begin{abstract}
We give the first polynomial-time algorithm for performing linear or polynomial regression resilient to adversarial corruptions in both examples and labels.

Given a sufficiently large (polynomial-size) training set drawn i.i.d. from distribution ${\cal D}$ and subsequently corrupted on some fraction of points, our algorithm outputs a linear function whose squared error is close to the squared error of the best-fitting linear function with respect to ${\cal D}$, assuming that the marginal distribution of $\cD$ over the input space is \emph{certifiably hypercontractive}. This natural property is satisfied by many well-studied distributions such as Gaussian, strongly log-concave distributions and, uniform distribution on the hypercube among others.  We also give a simple statistical lower bound showing that some distributional assumption is necessary to succeed in this setting.

These results are the first of their kind and were not known to be even information-theoretically possible prior to our work.  

Our approach is based on the sum-of-squares (SoS) method and is inspired by the recent applications of the method for parameter recovery problems in unsupervised learning. Our algorithm can be seen as a natural convex relaxation of the following conceptually simple non-convex optimization problem: find a linear function and a large subset of the input corrupted sample such that the least squares loss of the function over the subset is minimized over all possible large subsets.

\end{abstract}

%% file: content/intro.tex
\section{Introduction}

An influential recent line of work has focused on developing {\em robust} learning algorithms--
algorithms that succeed on a data set that has been
contaminated with adversarially corrupted outliers. It has led to important
achievements such as efficient algorithms for robust clustering and estimation of moments \citep{DBLP:journals/corr/LaiRV16,DBLP:journals/corr/DiakonikolasKKL16,DBLP:conf/stoc/CharikarSV17,DBLP:journals/corr/abs-1711-11581,DBLP:journals/corr/abs-1711-07465} in unsupervised learning and efficient learning of halfspaces \citep{DBLP:journals/jmlr/KlivansLS09,DBLP:journals/corr/DiakonikolasKS17} with respect to
malicious or ``nasty noise'' in classification.  In this paper, we continue this line of work and give the first
efficient algorithms for performing outlier-robust least-squares {\em regression}.  That is,
given a training set drawn from distribution ${\cal D}$ and arbitrarily corrupting an $\eta$
fraction of its points (by changing both labels and/or locations), our goal is to efficiently find a linear function (or polynomial in
the case of polynomial regression) whose least squares loss is competitive with the
best fitting linear function for ${\cal D}$.

We give simple examples showing that unlike classical regression, achieving any non-trivial guarantee for robust regression is information-theoretically impossible without making assumptions on the distribution $\cD$. In this paper, we study the case where the marginal of $\cD$ on examples in the well-studied class of \emph{hypercontractive} distributions. Many natural distributions such as Gaussians, strongly log-concave distributions, and product distributions on the hypercube with bounded marginals fall into this category. 

\subsection{Outlier-Robust Regression}
We formally define the problem next. In the following, we will use the following notations for brevity: For a distribution $\cD$ on $\R^d \times \R$ and for a vector $\ell \in \R^d$, let $\err_\cD(\ell) = \E_{(x,y) \sim \cD}[ (\iprod{\ell,x} - y)^2]$ and let $\opt(\cD) = \min_{\ell \in \R^d} \err_{\cD}(\ell)$ be the least error achievable. 

In the classical least-squares linear regression problem, we are given access to i.i.d. samples from a distribution $\cD$ over $\R^d \times \R$ and our goal is to find a linear function $\ell:\R^d \to \R$ whose squared-error---$\err_{\cD}(\ell)$---is close to the best possible, $\opt(\cD)$. 

In outlier-robust regression, our goal is similar with the added twist that we only get access to a sample from the distribution $\cD$ where up to an $\eta$ fraction of the samples have been arbitrarily corrupted.
\begin{definition}[$\eta$-Corrupted Samples]
Let $\cD$ be a distribution on $\R^d \times \R$. We say that a set $U \subseteq \R^d \times \R$ is an $\eta$-corrupted training set drawn from $\cD$ if it is formed in the
following fashion: generate a set $X$ of i.i.d samples from ${\cal D}$ and arbitrarily modify any $\eta$ fraction to produce $U$. 
\end{definition}
Observe that the corruptions can be \emph{adaptive}, that is, they can depend on the original uncorrupted sample $X$ in an arbitrary way as long as $|U \cap X | /|X| \geq 1 - \eta$.\footnote{In unsupervised learning, this has been called the \emph{strong adversary} model of corruptions and is the strongest notion of robustness studied in the context.}

Our goal---which we term \emph{outlier-robust regression}---now is as follows: Given access to an $\eta$-corrupted training set $U$ drawn from $\cD$, find a linear function $\ell$ whose error $\err_{\cD}(\ell)$ under the true distribution $\cD$ is small. 
\ignore{
\begin{definition}[$\epsilon$-corruption]
Let $Z = \{(x_i, y_i): 1 \leq i \leq n\}$ be an i.i.d. sample drawn from the distribution $\cD$. An $\epsilon$-corruption of $Z$ is an arbitrary subset $Z' = \{ (x_i',y_i') : 1 \leq i \leq n\} \subseteq \R^d\times \R$ such that $|Z' \cap Z| \geq (1-\epsilon)n.$
\end{definition}

We are now ready to define the robust polynomial regression problem for a distribution $\cD$. 
\begin{problem}
Let $D$ be any distribution on $\R^d \times \R.$ In the degree $t$ outlier-robust polynomial regression problem, we are given an $\epsilon > 0$ and a sample $Z'$ of size $n$ that is produced by an $\epsilon$ corruption of an i.i.d. sample of a distribution $\cD'$ whose marginal on the first $d$ coordinates equals $D.$ The goal is to output a polynomial $p$ such that:
\[
\E_{(x,y) \sim \cD}[ |p(x)-y|^2] < \min_{q: \text{ deg t polynomial } } \E_{(x,y) \sim \cD} [ |q(x)-y|^2] + \delta,
\]
\ignore{
\[
\E_{(x,y) \sim \cD'}[ |p(x)-y|] < \min_{q: \text{ deg t polynomial } } \E_{(x,y) \sim \cD'} [ |q(x)-y|] + \delta,
\]}
for as small a $\delta$ as possible.
\end{problem}}

\subsection{Statement of Results}
Our main results give outlier-robust least-squares regression algorithms for  hypercontractive distributions.  
\begin{definition}[$4$-Hypercontractivity]
A distribution $D$ on $\R^d$ is $(C,4)$-hypercontractive if for all $\ell \in \R^d$, $\E_{x \sim D}[\iprod{x,\ell}^4] \leq C^2 \cdot \E_{x \sim D}[\iprod{x,\ell}^2]^2$. 

In addition, we say that $D$ is \emph{certifiably} $(C,4)$-hypercontractive if there is a degree $4$ \emph{sum-of-squares proof} of the above inequality.
\end{definition}
Observe that $4$-hypercontractivity is invariant under arbitrary affine transformation, and in particular, doesn't depend on the condition number of the covariance of the distribution.

We will elaborate on the notion of \emph{certifiability} later on (once we have the appropriate preliminaries). For the time being, we note that many well-studied distributions including (potentially non-spherical) Gaussians, affine transformations of isotropic strongly log-concave distributions, the uniform distribution on the Boolean hypercube, and more generally, product distributions on bounded domains are known to satisfy this condition with $C$ a fixed constant. 


\begin{theorem}\label{th:intro4}[Informal]
Let ${\cal D}$ be a distribution on $\R^{d} \times [-M,M]$ and let ${\cal D_{X}}$ be its
marginal distribution on $\R^d$ which is certifiably $(C,4)$-hypercontractive. Let $\ell^* = \arg\min_{\ell} \err_{\cD}(\ell)$ have polynomial bit-complexity. Then for all $\epsilon > 0$ and $\eta < c/C^2$ for a universal constant $c > 0$,  there exists an algorithm $\cA$ with run-time $\poly(d,1/\eta,1/\epsilon,M)$ that given a polynomial-size $\eta$-corrupted training set $U$, outputs a linear function $\ell$ such that with probability at least  $1-\epsilon$, 
$$\err_{\cD}(\ell) \leq (1 + O(\sqrt{\eta})) \cdot \opt(\cD)  + O(\sqrt{\eta}) \E_{(x,y) \sim \cD}[(y - \iprod{\ell^*,x})^4] + \epsilon.$$

\end{theorem}

The above statement assumes that the marginal distribution is (certifiably) hypercontractive with respect to its fourth moments.  Our results improve for higher-order certifiably hypercontractive distributions $\cD_{\cX}$; see Theorem \ref{thm:analysis-L2-linear-regression} for details.  In the \emph{realizable case} where $(x,y) \sim \cD$ satisfies $y = \iprod{\ell^*,x}$ for some $\ell^*$, the guarantee of Theorem \ref{th:intro4} becomes $\err_\cD(\ell) \leq \epsilon$; in particular, the error approaches zero at a polynomial rate. In Section~\ref{sec:lower-bounds}, we give a simple example to show that distributional assumptions are necessary in the outlier-robust setting to get a finite bound on the error. 

We also get analogous results for outlier-robust polynomial regression. See Theorem~\ref{thm:polyregression}.

We believe that the dependence of the error on $\eta$ is likely suboptimal\footnote{A previous version of this paper had an erroneous  claim about an information-theoretic lower bound on the error of any estimator as a function of $\eta$. This was due to an issue in the analysis of the distribution we had constructed for the purpose of the lower bound. This was pointed out to us by Ainesh Bakshi and Adarsh Prasad. }. 
Finding an efficent algorithm for outlier-robust regression with an improved/right dependence on $\eta$ is an outstanding open problem. 

Our result is a outlier-robust analog of \emph{agnostic} regression problem - that is, the \emph{non-realizable} setting. 
In addition, our guarantees makes no assumption about the condition number of the covariance of $\cD_X$ and thus, in particular, holds for $\cD_X$ with poorly conditioned covariances. Alternately, we give a similar guarantee for $\ell_1$ regression when the condition number of covariance of $\cD_X$ is bounded without any need for hypercontractivity (see Theorem~\ref{thm:L1-regression-analysis}). We show that in the absence of distributional assumptions (such as hypercontractivity) it is statistically impossible to obtain any meaningful bounds on robust regression in Section \ref{sec:lower-bounds}.

 \paragraph{Application to Learning Boolean Functions under Nasty Noise}  Our work has immediate applications for learning Boolean functions in the {\em nasty noise} model, where the learner is presented with an $\eta$-corrupted training set that is derived from an uncorrupted training set of the form $(x,f(x))$ with $x$ drawn from ${\cal D}$ on $\{0,1\}^n$ and $f$ is an unknown Boolean function.  The goal is to output a hypothesis $h$ with $\Pr_{x}[h(x) \neq f(x)]$ as small as possible. The nasty noise model is considered the most challenging noise model for classification problems in computational learning theory. {}

Applying a result due to \citet{DBLP:journals/siamcomp/KalaiKMS08} (c.f. Theorem 5) for learning with respect to adversarial {\em label noise only} (standard agnostic learning) and a generalization of Theorem \ref{th:intro4} to higher degree polynomials (see Theorem~\ref{thm:polyregression}) we obtain the following:


\begin{corollary} \label{cor:nasty}
Let ${\cal C}$ be a class of Boolean functions on $n$ variables such that for every $c \in {\cal C}$ there exists a (multivariate) polynomial $p$ of degree $d(\eps)$ with $\E_{x \sim D}[(p(x) - c(x))^2] \leq \epsilon$.  Assume that $d(\eps)$ is a constant for any $\eps = O(1)$ and that ${\cal D}$ is $(C,4)$ hypercontractive for polynomials of degree $d(\eps^2)$.  Then ${\cal C}$ can be learned in the nasty noise model in time $n^{O(d(\eps^2))}$ via an output hypothesis $h$ such that $\Pr_{x \sim {\cal D}}[h(x) \neq c(x)] \leq O(\sqrt{\eta}) \E_{x \sim D}[(p(x) - c(x))^4] + \epsilon$.  
\end{corollary}

One of the main conclusions of work due to~\citet{DBLP:journals/siamcomp/KalaiKMS08} is that the existence of low-degree polynomial approximators for a concept class ${\cal C}$ implies learnability for ${\cal C}$ in the agnostic setting.  Corollary~\ref{cor:nasty} shows that existence of low-degree polynomial approximators {\em and} hypercontractivity of $D$ imply learnability in the harsher nasty noise model. 

We note that Corollary~\ref{cor:nasty} gives an incomparable set of results in comparison to recent work of~\citet{DBLP:journals/corr/DiakonikolasKS17} for learning polynomial threshold functions in the nasty noise model. 

\paragraph{Concurrent Works}
Using a set of different techniques, Diakonikolas, Kamath, Kane, Li, Steinhardt and Stewart \cite{DKKLSS18} and Prasad, Suggala, Balakrishnan and Ravikumar \cite{2018arXiv180206485P} also obtained robust algorithms for regression in the setting where data $(x,y)$ is generated via the process: $y = \iprod{w,x} + e$ for an fixed unknown vector $w$ and zero mean noise $e$.   For improved bounds for the case when $x$ is distributed according to a Gaussian, see recent (independent and concurrent) work due to Diakonikolas, Kong, and Stewart \cite{DKS18}.   

\subsection{Our Approach} \label{sec:overview}
In this section, we give an outline of Theorem \ref{th:intro4}. At a high level, our approach resembles several recent works~\citep{DBLP:journals/corr/MaSS16,DBLP:conf/colt/BarakM16,DBLP:conf/colt/PotechinS17,DBLP:journals/corr/abs-1711-11581,HopkinsLi17} starting with the pioneering work of~\citet{MR3388192-Barak15} that use the Sum-of-Squares method for designing efficient algorithms for learning problems. An important conceptual difference, however, is that previous works have focused on \emph{parameter recovery} problems. For such problems, the paradigm involves showing that there's a simple (in the ``SoS proof system'') proof that a small sample \emph{uniquely} identifies the underlying hidden parameters (referred to as ``identifiability'') up to a small error.  

In contrast, in our setting, samples do not uniquely determine a good hypothesis as there can be multiple hypotheses (linear functions) that all have low-error on the true distribution. Our approach thus involves establishing that there's a ``simple'' proof that \emph{any} low-error hypotheses that is inferred from the observed (corrupted) sample has low-error on the true distribution (we call this \emph{certifiability} of a good hypothesis). To output a good solution in our approach (unlike in cases where there are uniqueness results), we have to crucially rely on the convexity (captured in the SoS proof system) of the empirical loss function.  


\paragraph{Part One: \emph{Certifying} that a linear function has low loss}


Let $X$ be an uncorrupted sample from the underlying distribution $\cD$ and suppose we are given an $\eta$-corruption $U = \{(u_1,v_1), (u_2,v_2), \ldots, (u_n,v_n)\}$ of $X$. Let $\hat{\cD}$\footnote{We use superscript $\hat{\;}$ to denote empirical quantities and superscript $'$ to denote quantities on corrupted samples.} be the uniform distribution on $X$. Our goal is to come up with a linear function $\ell$ that has low error on $\hat{\cD}$ given access only to $U$. By standard generalization bounds, this will also imply that $\ell$ has low error on $\cD$ with high probability. 

It is important to observe that even without computational constraints, that is, \emph{information theoretically}, it is unclear why this should at all be possible. To see why, 
let's consider the following natural strategy: brute-force search over all subsets $T$ of $U$ of size
$(1 - \eta) |U|$ and perform least-squares regression to obtain linear
function $\ell_T$ with empirical loss $\epsilon_T$.  Then,
output $\ell_T$ with minimal empirical loss $\epsilon_T$ over all
subsets $T$.  

Since some subset $T^{*}$ of size $(1-\eta)|U|$ will be a proper subset of the uncorrupted sample, the empirical loss of $\ell_{T^{*}}$ will clearly be small. However, a priori, there's nothing to rule out the existence of another subset $R$ of size $(1-\eta)|U|$ such that the optimal regression hypothesis $\ell_R$ on $R$ has loss smaller than that of $\ell_{T^{*}}$ while $\ell_R$ has a large error on the $\hat{\cD}$. 
%

This leads to the following interesting question on \emph{certifying a good hypothesis}: given a linear function $\ell$ that has small empirical loss with respect to some subset $T$ of $(1-\eta)$ fraction of the corrupted training set $U$, can we {\em certify} that its {\em true} loss with respect to $X$ is small?

We can phrase this as a more abstract ``robust certification'' question: given two distributions $\cD_1$ (=uniform distribution on $X$ above) and $\cD_2$ (=uniform distribution on $T$ above) on $\R^d \times \R$ that are $\eta$ close in total variation distance, and a linear function $\ell$ that has small error on $\cD_2$, when can we certify a good upper bound on the error of $\ell$ on $\cD_1$? 

Without making any assumptions on $\cD_1$, it is not hard to construct examples where we can give no meaningful bound on the error of a good hypothesis $\ell$ on $\cD_1$ (see Section \ref{sec:lower-bounds}). More excitingly, we show an elementary proof of a ``robust certifiability lemma'' that proves a statement as above whenever $\cD_1$ has \emph{hypercontractive} one dimensional marginals. The loss with respect to ${\cal D}_1$ increases as a
function of the statistical distance and the degree of hypercontractivity.

Applying our certification lemma, it thus suffices to find a subset $T$ of $U$ of size $\geq (1-\eta)|U|$ and a linear function $\ell$ such that the least squares error of $\ell$ over $T$ is small. 
\paragraph{Part Two: Inefficient Algorithm via Polynomial Optimization}
Coming back to the question of efficient algorithms, the above approach can appear hopeless in general since simultaneously finding $\ell$ and a subset $T$ of size $(1-\eta)|U|$ that minimizes the error of $\ell$ w.r.t. uniform distribution on $T$ is a non-convex quadratic optimization problem. At a high-level, we will be able to get around this intractability by observing that the \emph{proof} of our robust certifiability lemma is ``simple'' in a precise technical sense. This simplicity allows us to convert such a certifiability proof into an efficient algorithm in a principled manner. To describe this connection, we will first translate the naive idea for an algorithm above into a polynomial optimization problem. 

For concreteness in this high-level description, we suppose that for $(x,y) \sim \cD$, the distribution on $x$ is $(C,4)$-hypercontractive for a fixed constant $C$ and $\E[y^4] = O(1)$. Further, it can also be shown that, with high probability, $\hat{\cD}$ is also $(O(1), 4)$-hypercontractive as long as the size of the original uncorrupted sample $X$ is large enough.

Following the certification lemma, our goal is to use $U$ to find a distribution $\cD'$ and a linear function $\ell$ such that 1) the loss of $\ell$ with respect to $\cD'$ is small and 2) $\cD'$ is close to $\widehat{\cD}$. It is easy to phrase this as a polynomial optimization problem.

To do so we will look for $X' = \{(x_1',y_1'),\ldots,(x_n',y_n')\}$ and \emph{weights} $w_1,w_2,\ldots, w_n \in \{0,1\}$ with $\sum_i w_i \geq (1- \eta) n$ and $(x_i',y_i') = (u_i,v_i)$ if $w_i = 1$. Let $\cD'$ be the uniform distribution on $X'$. Clearly, the condition on weights $w$ ensures that the statistical distance between $\hat{\cD}, \cD'$ is at most $\eta$.  Ideally, we intend $w_i$'s to be the indicators of whether or not the $i$'th sample is corrupted. We now try to find $\ell$ that minimizes the least squares error on $\cD'$. This can be captured by the following optimization program: $\min_{w,\ell,X'} (1/n) \sum_i (y_i' - \iprod{\ell,x_i'})^2$ where $(w,\ell,X')$ satisfy the polynomial system of constraints:

\begin{equation}
  \cP = 
  \left \{
    \begin{aligned}
      & \textstyle\sum_{i=1}^n w_i
      = (1-\eta) \cdot n & &\\
      & w_i^2
      =w_i 
      &\forall i\in [n]. &\\
      & w_i \cdot (u_i - x'_i)
       = 0
      &\forall i\in [n]. &\\
      & w_i \cdot (v_i - y'_i)
      = 0
      &\forall i\in [n]. &
    \end{aligned}
  \right \} \label{eq:introprogram}
\end{equation}

In this notation, our robust certifiability lemma implies that for any $(w,\ell,X')$ satisfying $\cP$, 
\begin{equation}\label{eq:intro1}
\err_{\hat{\cD}}(\ell) \leq (1 + O(\sqrt{\eta})) \cdot \err_{\cD'}(\ell) + O(\sqrt{\eta}).
\end{equation}

It is easy to show that the minimum of the optimization program $\opt(\widehat{\cD}) \lessapprox \opt(\cD)$ (up to standard generalization error) by setting $X' = X$ and $w_i = 1$ if and only if $i$'th sample is uncorrupted. By the above arguments, solutions to the above program satisfy the bound stated in Theorem \ref{th:intro4}. Unfortunately, this is a quadratic optimization problem and is NP-hard in general. 

We are now ready to describe the key idea that allows us to essentially turn this hopelessly inefficient algorithm into an efficient one. This exploits a close relationship between the simplicity of the proof of robust certifiability and the success of a canonical semi-definite relaxation of \eqref{eq:introprogram}.

\paragraph{Part Three: From Simple Proofs to Efficient Algorithms} 

Suppose that instead of finding a single solution to the program in \eqref{eq:introprogram}, we attempt to find a distribution $\mu$ supported on $(w,\ell,X')$ that satisfy $\cP$ and minimizes $\E_\mu[(1/n) \sum_i (y_i' - \iprod{\ell,x'_i})^2]$. Let $\opt_{\mu}$ be the minimum value. Then, as Equation \ref{eq:intro1} holds for all $(w,\ell,X')$ satisfying $\cP$, it also follows that
\begin{equation}\E_{(w,\ell,X') \sim \mu}[ \err_{\hat{\cD}}(\ell)] \leq (1 + O(\sqrt{\eta})) \opt_\mu + O(\sqrt{\eta}).\label{eq:error-equation-intro}\end{equation}

A priori, we appear to have made our job harder. While computing a distribution on solutions is no easier than computing a single solution, even describing a distribution on solutions appears to require exponential resources in general. However, by utilizing the convexity of the square loss, we can show that having access to just the first moments of $\mu$ is enough to recover a good solution. 

Formally, by the convexity of the square loss, the above inequality yields:
\begin{equation} \label{eq:convexity-intro}\err_{\hat{\cD}}\left(\E_\mu[\ell]\right) \leq \E_{(w,\ell,X') \sim \mu}[ \err_{\hat{\cD}}(\ell)] \leq (1 + O(\sqrt{\eta})) \opt_\mu + O(\sqrt{\eta}).\end{equation}

All of the above still doesn't help us in solving program \ref{eq:introprogram} as even finding first moments of distributions supported on solutions to a polynomial optimization program is NP-Hard. 


The key algorithmic insight is to observe that we can replace distributions $\mu$ by an efficiently computable (via the SoS algorithm) proxy called as \emph{pseudo-distributions} without changing any of the conclusions above. 

In what way is a pseudo-distribution a proxy for an actual distribution $\mu$ satisfying $\cP$? It turns out that if a polynomial inequality (such as the one in \eqref{eq:intro1}) can be derived from $\cP$ via a \emph{low-degree sum-of-squares} proof, then \eqref{eq:error-equation-intro} remains valid even if we replace $\mu$ in \eqref{eq:error-equation-intro} by a pseudo-distribution $\tmu$ of large enough degree. Roughly speaking, the SoS degree of a proof measures the ``simplicity'' of the proof (in the ``SoS proof system''). In other words, facts with simple proofs holds not just for distributions but also for pseudo-distributions.   

Thus, the important remaining steps are to show that 1) the inequality \eqref{eq:intro1} (which is essentially the conclusion of our robust certifiability lemma) and 2) the convexity argument in \eqref{eq:convexity-intro} has a low-degree SoS proof. We establish both these claims by relying on standard tools such as the SoS versions of the Cauchy-Schwarz and H\"older's inequalities. 

We give a brief primer to the SoS method in Section~\ref{sec:sospreliminaries} that includes rigorous definitions of concepts appearing in this high-level overview.

\ignore{
\subsection{Part Two of Our Approach: Enter SoS}
For concreteness in this high-level description, we suppose that for $(x,y) \sim \cD$, the distribution on $x$ is $(C,4)$-hypercontractive for a fixed constant $C$ and $\E[y^4] = O(1)$. 

Now suppose we are given an $\eta$-corrupted training set $S = \{(x_1,y_1),\ldots,(x_n, y_n)\}$ of size $n$ drawn from the distribution $\cD$. Just for analysis, let $T$ be the underlying uncorrupted training set from which $S$ is obtained. Let $\hat{\cD}$ be the uniform distribution on $T$\footnote{We use superscript $\hat{\;}$ to denote empirical quantities and superscript $\tilde{\;}$ to denote quantities on corrupted samples.}. If we can find a linear function $\ell$ with small error under $\hat{\cD}$, then standard generalization arguments would imply that the loss of $\ell$ with respect to $\cD$ is small for $n$ sufficiently big. Further, it can also be shown that for $n \gg d$, with high probability, $\hat{\cD}$ is also $(O(1), 4)$-hypercontractive. 

Thus, given the certification lemma, the problem reduces to finding a near-uniform distribution $\widetilde{\cD}$ on $S$ and a linear function $\ell$ such that the loss of $\ell$ with respect to $\widetilde{\cD}$ is small. Concretely, suppose $\widetilde{\cD}$ be defined as follows: choose \emph{weights} $w_1,\ldots,w_n \in \{0,1\}$ with $\sum_i w_i \geq (1-\eta) n$ and let $\widetilde{\cD}$ be distribution of random variable sampled as follows: Output $(x_i,y_i)$ with probability proportional to $w_i$. Clearly, the statistical distance between $\hat{\cD}$ and $\widetilde{\cD}$ is at most $2 \eta$. 

Thus, the problem reduces to finding weights $w$ as above and a linear function $\ell$ such that the loss of $\ell$ with respect to the corresponding distribution $\widetilde{\cD}$ is small. Such a $w$ clearly exists: we could give zero weight to the corrupted points in $S$. 

More specifically, let 
$$\cA = \left\{ (w,\ell) : w_i^2 = w_i, \forall i =1,\ldots,n, \; \sum_{i=1}^n w_i \geq (1-\eta) n\right\}.$$
Then, our robust certifiability lemma implies that for any $(w,\ell) \in \cA$, 
\begin{equation}\label{eq:intro1}
\err_{\hat{\cD}}(\ell) \leq (1 + O(\sqrt{\eta})) \cdot \left(\frac{1}{n}\sum_{i=1}^n w_i (y_i - \iprod{\ell,x_i})^2\right) + O(\eta) .
\end{equation}

Further, $\min_{w,\ell \in \cA} \sum_{i=1}^n w_i (y_i - \iprod{\ell,x_i})^2 \approx \opt(\cD)$ (up to standard generalization error) as can be achieved by giving zero weight to the corrupted points in $S$. 

As such, we could write the following optimization progam to
find $w$ and $\ell$:
\begin{align}\label{eq:programintro}
&\min_{w,\ell \in \cA}\;\;\frac{1}{n} \sum_{i=1}^n w_i (y_i - \iprod{\ell,x_i})^2&\\
& \cA = \left\{ (w,\ell) : w_i^2 = w_i, \forall i =1,\ldots,n, \; \sum_{i=1}^n w_i \geq (1-\eta) n\right\}&\nonumber
\end{align}

By the above arguments, solutions to the above program satisfy what we want: $\err_{\hat{\cD}}(\ell) \approx (1+ O(\sqrt{\eta})) \opt(\cD) + O(\sqrt{\eta})$ - the bound stated in Theorem \ref{th:intro4}. Unfortunately, it is NP-hard to solve optimization problems as above in general. We get around this hurdle by considering the sum of squares (SOS) relaxation of the program and arguing that the relaxation gives us a $\ell$ such that $\err_{\hat{\cD}}(\ell)$ is similarly small. 

\subsection{Parth Three of Our Approach: Exploiting SOS Proofs of Robust Certifiability}
Suppose that in lieu of solving the program on Equation \ref{eq:programintro}, we have access to a density $\mu$ on $(w,\ell) \in \cA$. Let $\opt_\mu = \E_{(w,\ell) \sim \mu}[ (1/n) \sum_i w_i (y_i - \iprod{\ell,x_i})^2]$ be the expected output of the optimization program under the distribution $\mu$. Then, as Equation \ref{eq:intro1} holds for all $(w,\ell)$, it also follows that
$$\E_{(w,\ell) \sim \mu}[ \err_{\hat{\cD}}(\ell)] \leq (1 + O(\sqrt{\eta})) \opt_\mu + O(\sqrt{\eta}).$$

All of the above does not really help us in solving program \ref{eq:programintro} as if we had access to a density $\mu$ as above with $\opt_\mu \approx \opt(\cD)$, we could have just sampled from it in the first place. Nevertheless, the above view is useful because of the following simple fact: owing to convexity of squared-loss, the above inequality also implies that 
$$\err_{\hat{\cD}}\left(\E_\mu[\ell]\right) \leq \E_{(w,\ell) \sim \mu}[ \err_{\hat{\cD}}(\ell)] \leq (1 + O(\sqrt{\eta})) \opt_\mu + O(\sqrt{\eta}).$$

Thus, at the very least, instead of sampling access to $\mu$, it suffices to know the \emph{first moments} of $\ell$ under the distribution $\mu$. Unfortunately, even this seems computationally intractable. 

Fortunately, we can replicate the above approach efficiently by working with the SOS hierarchy and in particular letting $\widetilde{\mu}$ be a \emph{pseudo-density} on $\cA$ (instead of a density on $\cA$) and working with \emph{pseudo-expectations} under $\widetilde{\mu}$ (instead of expectations under $\mu$). We could then potentially use the pseudo-expectation of $\ell$ as our final candidate hypothesis. 

To analyze this approach we show that our robust certifiabilty lemma as well as the convexity-argument mentioned above can be captured in the SOS proof system (with the latter being immediate). Given this, we can finish the analysis of our algorithm by utilizing known properties of the SOS hierarchy.  }

\ignore{

We defer the formal details of the SOS hierarchy to Section \ref{sec:sos} and assume familiarity with the notion of \emph{pseudo-distributions} and \emph{pseudo-expectations} for this high-level description. Our analysis exploits the following abstract property of SOS and pseudo-densities: If $\cB = \{y: p_i(y) \geq 1 \leq i \leq r\}$ is a set defined by polynomial inequalities (i.e., $p_i$ are polynomials) such that these inequalities SOS-imply another polynomial inequality $q(y) > 0$, then for any pseudo-density $\mu$ supported on $\cB$, the corresponding pseudo-density $\pE_\mu$ also satisfies $q$, meaning $\pE_\mu[q] > 0$. 

Now, note that the set $\cA$ is defined by degree two polynomial inequalities. We show that our robust certifiability lemma has a SOS proof; that is, viewing Equation \ref{eq:intro1} as a polynomial inequality in variables $w,\ell$, this inequality is SOS implied by $\cA$. Thus, by the above observation, any pseudo-density $\mu$ with sufficiently high degree on $\cA$ also satisfies an analogue of  Equation \ref{eq:intro1}:
\begin{equation}\label{eq:introps1}
\pE_\mu[\err_{\hat{\cD}}(\ell)] \leq (1 + O(\sqrt{\eta})) \cdot \pE_\mu\left[\left(\sum_{i=1}^n w_i (y_i - \iprod{\ell,x_i})^2\right)\right] + O(\eta) .
\end{equation}
Thus, if we optimize over pseudo-densities $\mu$ on $\cA$ and minimize $\pE_\mu\left[\left(\sum_{i=1}^n w_i (y_i - \iprod{\ell,x_i})^2\right)\right]$, then we get a pseudo-density $\mu^*$ such that 
$$\pE_{\mu^*}[\err_{\hat{\cD}}(\ell)] \approx (1 + O(\sqrt{\eta})) \opt(\cD) + O(\eta).$$

As it stands, the above inequality is not very useful for us as it does not tell us which $\ell$ to choose. We next exploit the convexity of the loss function $\err_{\hat{\cD}}(\;\;)$ to show that for $\tilde{\ell} = \pE_{\mu^*}[\ell]$, 
$$\err_{\hat{\cD}}(\tilde{\ell}) = \err_{\hat{\cD}}(\pE_{\mu^*}[\ell]) \leq \pE_{\mu^*}[\err_{\hat{\cD}}(\ell)] \approx (1 + O(\sqrt{\eta})) \opt(\cD) + O(\eta),$$
finishing the argument.}

\subsection{Related Work}
The literature on grappling with outliers in the context of regression is vast, and we do not attempt a survey here\footnote{Even the term ``robust'' is very overloaded and can now refer to a variety of different concepts.}.  Many heuristics have been developed modifying the ordinary least squares objective with the intent of minimizing the effect of outliers (see \citet{MR914792}).  Another active line research is concerned with {\em parameter recovery}, where each label $y$ in the training set is assumed to be from a generative model of the form $\theta^{T} x + e$ for some (usually independent) noise parameter $e$ and unknown weight vector $\theta \in \R^d$.  For example, the recovery properties of LASSO and related algorithms in this context have been intensely studied (see e.g., \citet{DBLP:journals/tit/XuCM10}, \citet{LoW07}). For more challenging noise models, recent work due to Du, Balakrishnan, and Singh~\citep{DuBS17} studies sparse recovery in the Gaussian generative setting in Huber's $\epsilon$-contamination model, which is similar but formally weaker than the noise model we consider here.  

It is common for ``robust regression'' to refer to a scenario where only the labels are allowed to be corrupted adversarially (for example, see~\citet{DBLP:conf/nips/Bhatia0KK17} and the references therein), or where the noise obeys some special structure (e.g., \citet{HermanS10}) (although there are some contexts where both the covariates (the $x$'s) and labels may be subject to a small adversarial corruption~\citep{ChenCM13}). 

What distinguishes our setting is 1) we do not assume the labels come from a generative model; each $(x,y)$ element of the training set is drawn iid from ${\cal D}$ and 2) we make no assumptions on the structure or type of noise that can affect a training set (other than that at most an $\eta$ fraction of points may be affected).  In contrast to the parameter recovery setting, our goal is similar to that of {\em agnostic learning}: we will output a linear function whose squared error with respect to ${\cal D}$ is close to optimal.    

From a technical standpoint, as discussed before our work follows the recent paradigm of converting certifiability proofs to algorithms. Previous applications in machine learning have focused on various parameter-recovery problems in unsupervised learnings. Our work is most closely related to the recent works on robust unsupervised learning (moment estimation and clustering)~\citep{DBLP:journals/corr/abs-1711-11581,HopkinsLi17,KothariSteinhardt17}. 
\ignore{

\begin{itemize}

\item Regression with respect to pure label noise; Lasso; Prateek Jain
  has a recent paper on this. 

\item ``Computationally efficient robust estimation of sparse
  functionals'' -- This paper seems to have results similar to Ilias
  robust mean estimation papers (unsupervised learning).  Then it also
  includes results on sparse linear regression, but the results for
  sparse linear regression don't seem to be robust in any way. Sort of
  weird they include it in the paper.

\item Might want to mention the paper ``Robust Regression and Lasso''
  where they discuss something they call Robust Linear Regression
  (though it seems to not handle label noise). 

\end{itemize}
}

%% file: content/setup.tex
\section{Preliminaries and Notation}

\subsection{Notation}
We will use the following notations and conventions throughout: For a distribution $\cD$ on $\R^d \times \R$ and function $f:\R^d \to \R$, we define $\err_\cD(f) = \E_{(x,y) \sim \cD}[ (f(x) - y)^2]$. For a vector $\ell \in \R^d$, we abuse notation and write $\err_\cD(\ell)$ for $\E_{(x,y) \sim \cD}[ (\iprod{\ell,x} - y)^2]$.  For a real-valued random variable $X$, and integer $k \geq 0$, we let $\|X\|_k = \E[X^k]^{1/k}$.

\subsection{Distribution Families}
Our algorithmic results for a wide class of distributions that include Gaussian distributions and others such as log-concave and other product distributions. We next define the properties we need for the marginal distribution on examples to satisfy. 
\begin{definition}[Certifiable hypercontractivity]\label{def:hyperconc1}
For a function $C:[k] \to \R_+$, we say a distribution $D$ on $\R^d$ is $k$-certifiably $C$-hypercontractive if for every $r \leq k/2$, there's a degree $k$ sum of squares proof of the following inequality in variable $v$:
\[
\E_D \iprod{x,v}^{2 r} \leq \Paren{C(r) \E_{D} \iprod{x,v}^{2}}^{r}.
\] 
\end{definition}

Many natural distribution families satisfy certifiable hypercontractivity with reasonably growing functions $C$. For instance, Gaussian distributions, uniform distribution on Boolean hypercube satisfy the definitions with $C(r) = c r$ for a fixed constant $c$. 
\Pnote{certifiable hypercontractivity of log-concave distributions is not known. It's true only under the KLS cooling conjecture so far.} More generally, all distributions that are affine transformations of isotropic distributions satisfying the Poincar\'{e} inequality \citep{DBLP:journals/corr/abs-1711-07465},  are also certifiably hypercontractive. In particular, this includes all strongly log-concave distributions.  Certifiable hypercontractivity also satisfies natural closure properties under simple operations such as affine transformations, taking bounded weight mixtures and taking products. We refer the reader to \citet{DBLP:journals/corr/abs-1711-11581} for a more detailed overview where certifiable hypercontractivity is referred to as certifiable subgaussianity.


%% file: content/overview.tex




%% file: content/identifiability.tex

\section{Robust Certifiability} \label{sec:robust-certifiability}

The conceptual core of our results is the following \emph{robust certifiability} result: for \emph{nice} distributions (e.g., as defined in Definition \ref{def:hyperconc1})
, a regression hypothesis inferred from a large enough $\epsilon$-corrupted sample has low-error over the uncorrupted distribution.

\subsection{Robust Certifiability for Arbitrary Distributions}
We begin by giving a robust certifiability claim for arbitrary distributions for L1 regression.

 The error that we incur depends on the L2 squared loss of the best fitting regression hypothesis, and in particular, we do not obtain \emph{consistency} in the statistical sense: i.e, the error incurred by the regression hypothesis does not vanish even in the ``realizable'' case when, in the true uncorrupted distribution, there's a linear function that correctly computes all the labels. In Section \ref{sec:lower-bounds}, we show that if we make no further assumption on the distribution, then this is indeed inherent and that achieving consistency under adversarial corruptions is provably impossible without making further assumptions. In the following subsection, we show that assuming that the moments of the underlying uncorrupted distribution are ``bounded'' (i.e., linear functions of the distribution are hypercontractive), one can guarantee consistency even under the presence of adversarial outliers.

 While the certifiability statements are independently interpretable, for the purpose of robust regression, it might be helpful to keep in mind that $D$ corresponds to uniform distribution on large enough sample from the unknown uncorrupted distribution while $D'$ corresponds to the uniform distribution on the sample that serves as the ``certificate''. 

\begin{lemma}[Robust Certifiability for L1 Regression]
Let $\cD,\cD'$ be two distributions on $\R^{d} \times \R$ with marginals $D, D'$ on $\R^d$, respectively.
Suppose $\|\cD-\cD'\|_{TV} \leq \eta$ and further, that the ratio of the largest to the smallest eigenvalue of the 2nd moment matrix of $D$ is at most $\kappa$. Then, for any $\ell,\ell^{*} \in \R^d$ such that $\|\ell^{*}\|_2 \geq \|\ell\|$,



\[
\E_{\cD} \abs{\iprod{\ell,x} - y} \leq \E_{\cD'} \abs{\iprod{\ell,x} -y} + 2\kappa^{1/2} \eta^{1/2} \sqrt{\E_{\cD} y^2} + 2 \kappa^{1/2} \eta^{1/2} \cdot \sqrt{\E_{\cD} (y - \iprod{\ell^{*},x})^2}\mper
\]

\end{lemma}

\begin{proof}
Let $G$ be a coupling between $\cD,\cD'$. That is, $G$ is a joint distribution on $(x,y), (x',y')$ such that the marginal on $(x',y')$ is $\cD'$ and the marginal on $(x,y)$ is $\cD$ satisfying $\Pr_G \1 \Set{ (x,y) = (x',y')} = 1-\eta.$ 
Let $\err_{\cD'}(\ell) = \E_{\cD'} \abs{y-\iprod{\ell,x}}$.
We have: 
\begin{align*}
\E_{\cD}  \abs{y-\iprod{\ell,x}} &= \E_G \1 \Set{ (x,y) = (x',y')} \abs{y-\iprod{\ell,x}} + \E_{G} \1 \Set{ (x,y) \neq (x',y')} \cdot \abs{y-\iprod{\ell,x}}\\
&\leq \err_{\cD'}(\ell) + \sqrt{\E_{G}\1 \Set{ (x,y) \neq (x',y')}^2} \sqrt{ \E_{\cD} (y-\iprod{\ell,x})^2}\\
&= \err_{\cD'}(\ell) + \sqrt{\eta} \sqrt{\E_{\cD} (y-\iprod{\ell,x})^2}\mper
\end{align*}

Now, we must have: $\E_{\cD} (y - \iprod{\ell,x})^2 \leq 2 \E_{\cD} y^2 + 2 \E_{\cD} \iprod{\ell,x}^2.$

For any $\ell^{*}$, $\E_{\cD} \iprod{\ell^{*},x}^2 \leq 2\E_{\cD} y^2 + 2\E_{\cD} (y-\iprod{\ell^{*},x})^2.$   

Since the all eigenvalues of $\E_{D} xx^{\top}$ are within $\kappa$ of each other and $\|\ell^{*}\|_2 \geq \|\ell\|$, $\E_{D} \iprod{\ell,x}^2 \leq \kappa \cdot \E_{D} \iprod{\ell^{*},x}^2$. Plugging in the above estimate gives the lemma.







\end{proof}

\subsection{Robust Certifiability for Hypercontractive Distributions}

The main result of this section is the following lemma.
\begin{lemma}[Robust Certifiability for L2 Regression]
Let $\cD,\cD'$ be distributions on $\R^d \times \R$ such that $\|\cD-\cD'\|_{TV} \leq \epsilon$ and the marginal $\cD_X$ of $\cD$ on $x$ is $k$-certifiably $C$-hypercontractive for some $C:[k] \to \R_+$ and for some even integer $k \geq 4$.

Then, for any $\ell,\ell^* \in \R^d$ and any $\eta$ such that $2 C(k/2) \eta^{1-2/k} < 0.9$, we have:
\[
\err_{\cD}(\ell) \leq (1+O(C(k/2)) \eta^{1-2/k}) \cdot \err_{\cD'}(\ell) + O(C(k/2))\eta^{1-2/k} \cdot \Paren{\E_\cD (y-\iprod{\ell^{*},x})^k}^{2/k}\mper\]
\label{lem:identifiability-least-squares-linear}

\end{lemma}
\begin{proof}
Fix a vector $\ell \in \R^d$; for brevity, we write $\err_\cD$ for $\err_\cD(\ell)$ and $\err_{\cD'}$ for $\err_{\cD'}(\ell)$ in the following. 

Let $G$ be a coupling between $\cD,\cD'$. That is, $G$ is a joint distribution on $(x,y), (x',y')$ such that the marginal on $(x',y')$ is $\cD'$ and the marginal on $(x,y)$ is $\cD$ satisfying $\Pr_G \1 \Set{ (x,y) = (x',y')} = 1-\eta.$ 

Let $((x,y), (x',y')) \sim \cG$. Writing $1 = \1 \Set{ (x,y) = (x',y')} + \1 \Set{ (x,y) \neq (x',y')}$, we obtain:
\begin{align}
\E_{\cD}[(y-\iprod{\ell,x})^2] &= \E_{\cG}[ \1 \Set{ (x,y) = (x',y')} (y-\iprod{\ell,x})^2] + \E_{\cG}[ \1 \Set{ (x,y) \neq (x',y')} \cdot (y-\iprod{\ell,x})^2]\notag\\
&= \E_{\cG} [\1 \Set{ (x,y) = (x',y')} (y' -\iprod{\ell,x'})^2 ]+  \E_{\cG}[ \1 \Set{ (x,y) \neq (x',y')} \cdot (y-\iprod{\ell,x})^2]\notag\\
&\leq \err_{\cD'} + \Paren{\E_{\cG}[\1 \Set{ (x,y) \neq (x',y')}^{k/k-2}]}^{1-2/k} \Paren{\E_{\cD} (y-\iprod{\ell,x})^k}^{2/k}\notag\\
&\leq \err_{\cD'}+ \eta^{1-2/k} \cdot \Paren{\E (y-\iprod{\ell,x})^k}^{2/k} \label{eq:CS-bound1}\mper
\end{align}

Here, the inequality uses the H\"older's inequality for the second term and the fact that $\E_\cG \1 \Set{ (x,y) = (x',y')} (y-\iprod{\ell,x})^2 \leq \E_{\cD'} (y-\iprod{\ell,x})^2 = \err_{\cD'}(\ell)$ for the first term. 

We next bound $\|y - \iprod{\ell,x}\|_k$. By Minkowski's inequality, 
$$\|y - \iprod{\ell,x}\|_k \leq \|y - \iprod{\ell^*,x}\|_k + \|\iprod{\ell - \ell^*,x}\|_k.$$
Now, by using hypercontractivity of $\cD_X$, we get
\begin{equation}\label{eq:gen1}
\|\iprod{\ell - \ell^*,x} \|_k \leq \sqrt{C(k/2)} \cdot \|\iprod{\ell - \ell^*,x}\|_2.
\end{equation}

Further, 
$$\|\iprod{\ell - \ell^*,x}\|_2 \leq \|y - \iprod{\ell^*,x}\|_2 + \|y - \iprod{\ell,x}\|_2 \leq \|y - \iprod{\ell^*,x}\|_k + \|y - \iprod{\ell,x}\|_2.$$

Combining the above three inequalities, we get
$$\|y - \iprod{\ell,x}\|_k \leq (1 + \sqrt{C(k/2)}) \|y - \iprod{\ell^*,x}\|_k + \sqrt{C(k/2)} \|y - \iprod{\ell,x}\|_2.$$
Therefore, as $(a+b)^2 \leq 2 a^2 + 2 b^2$ and $2 (1 + \sqrt{C(k/2)})^2 \leq 8 C(k/2)$, 
$$\|y - \iprod{\ell,x}\|_k^2  \leq 8 C(k/2) \|y - \iprod{\ell^*,x}\|_k^2 + 2 C(k/2) \err_\cD.$$
Substituting the above into Equation \ref{eq:CS-bound1}, we get
$$\err_{\cD} \leq \err_{\cD'} + 8 \eta^{1-2/k} C(k/2) \cdot  \|y - \iprod{\ell^*,x}\|_k^2 + 2 \eta^{1-2/k} C(k/2) \err_{\cD}.$$
Rearranging the inequality and observing that $1/(1- 2 \eta^{1-2/k} C(k/2)) \leq 1 + O(C(k/2)) \eta^{1-2/k}$ gives us
$$\err_{\cD} \leq (1 + O(C(k/2)) \eta^{1-2/k}) \err_{\cD'} +O(C(k/2)) \eta^{1-2/k} \cdot  \|y - \iprod{\ell^*,x}\|_k^2,$$
proving the claim.
\ignore{

For $\ell^{*} \in \arg \min_{\ell \in \R^d} \E_D (y-\iprod{\ell,x})^2$, observe that for any $\ell \in \R^d$, 

\begin{equation}
\E_D(y-\iprod{\ell,x})^2 = \E_D (y-\iprod{\ell^{*},x})^2 + \E_D \iprod{\ell-\ell^{*},x}^2\mper \label{eq:L2-fact}
\end{equation}

On the other hand, $\E_D(y-\iprod{\ell,x})^k = 2^{k-1}\E_D (y-\iprod{\ell^{*},x})^k + 2^{k-1}\E_D \iprod{\ell-\ell^{*},x}^k$.
Using that $D_s$ is $4$-certifiably $C$-subgaussian, we have:
$\E_D \iprod{\ell-\ell^{*},x}^k \leq (Ck)^{k/2} \Paren{\E_D \iprod{\ell-\ell^{*},x}^2}^{k/2}$.

Combining this with \eqref{eq:CS-bound}, we obtain that:
\[
(1-4Ck \eta^{1-2/k}) \E_D \iprod{\ell-\ell^{*},x}^2 + \E_D (y-\iprod{\ell^{*},x})^2 \leq \err_{D'} + 4\eta^{1-2/k} \Paren{\E_D (y-\iprod{\ell^{*},x})^k}^{2/k}
\]
Using \eqref{eq:L2-fact} again,
\begin{multline}
(1-4Ck\eta^{1-2/k})\E_D (y-\iprod{\ell,x})^2 \leq 4Ck \eta^{1-2/k} \E_D (y-\iprod{\ell^{*},x})^2 +\err_{D'} + 4\eta^{1-2/k} \paren{\E_D (y-\iprod{\ell^{*},x})^k}^{2/k} \\\leq \err_{D'} + 8Ck\eta^{1-2/k} \paren{\E_D (y-\iprod{\ell^{*},x})^k}^{2/k}
\end{multline}
Using that $8Ck\eta^{1-2/k} <0.9$, and thus, $\frac{1}{1-8Ck\eta^{1-2/k}} \leq 1+ 80Ck\eta^{1-2/k}$, we thus obtain the claim. }








\end{proof}

The argument for the above lemma also extends straightforwardly to polynomial regression (see Appendix~\ref{sec:app-moved}):

\ignore{
\begin{proof}

The proof is entirely analogous to the case of linear regression done previously.

As before, for brevity, we write $\err_D$ for $\err_D(P)$ and $\err_{D'}$ for $\err_{D'}(P)$ in the following. 

Let $G$ be a coupling between $D,D'$ such that $\Pr \1 \Set{ (x,y) = (x',y')} = 1-\epsilon$.

Writing $1 = \1 \Set{ (x,y) = (x',y')} + \1 \Set{ (x,y) \neq (x',y')}$, we obtain:
\begin{align}
\E_{G}  (y-\iprod{P,x^{\otimes t}})^2 &= \E_G \1 \Set{ (x,y) = (x',y')} (y-\iprod{P,x^{\otimes t}})^2 + \E_{G} \1 \Set{ (x,y) \neq (x',y')} \cdot (y-\iprod{P,x^{\otimes t}})^2\notag\\
&= \E_G \1 \Set{ (x,y) = (x',y')} (y' -\iprod{P,{x'}^{\otimes t}})^2 +  \E_{G} \1 \Set{ (x,y) \neq (x',y')} \cdot (y-\iprod{P,x^{\otimes t}})^2\notag\\
&\leq \err_{D'} + \sqrt{\E_{G}\1 \Set{ (x,y) \neq (x',y')}^2} \sqrt{ \E_{D'} (y-\iprod{P,{x'}^{\otimes t}})^4}\notag\\
&= \err_{D'}+ \sqrt{\epsilon} \sqrt{\E_{D} (y-\iprod{P,x^{\otimes t}})^4} \label{eq:CS-bound}\mper
\end{align}

Here, the inequality uses the Cauchy-Schwarz inequality for the second term and the fact that $\E_G \1 \Set{ (x,y) = (x',y')} ( y' - \iprod{P,{x'}^{\otimes t}})^2 \leq \E_{D'} (y'-\iprod{P,{x'}^{\otimes t}})^2 = \err_{D'}$ for the first term. 

Now, let us consider the case when $D$ is known to be $4$-certifiably $(C,t)$-hypercontractive. We then have: 
\[
\E_D (y - \iprod{P,x^{\otimes t}})^4 \leq 8 \E_D y^4 + 8 \E_D \iprod{P,x^{\otimes t}}^4.
\]

Using certifiable hypercontractivity of $D$, we have that: 
\[
\E_D \iprod{P,x^{\otimes t}}^4 \leq 2C\Paren{\E_D \iprod{P,x^{\otimes t}}^2}^2.
\]

Finally, by triangle inequality for $\ell_2$ norm, we have: 
\[
\sqrt{\E_D \iprod{P,x^{\otimes t}}^2} \leq \sqrt{\E_D (y-\iprod{P,x^{\otimes t}})^2} + \sqrt{ \E_D y^2} \leq \sqrt{\err_{D}} + \sqrt{\E_D y^2}.
\]

Thus, combining the above estimates and using the inequality that $(a+b)^4 \leq 8a^4 + 8b^4$, we obtain that:
\[
\E_D (y-\iprod{P,x^{\otimes t}})^4 \leq 8 \E_D y^4 + 16Ct (8 \err_{D}^2 + 8(\E_D y^2)^2 ).
\] 
Or,

\[
(\E_D (y-\iprod{P,x^{\otimes t}})^4)^{1/2} \leq 4\sqrt{\E_D y^4} + 16\sqrt{Ct} (\err_{D} + \E_D y^2).
\] 

Rearranging, this yields that $(1- 16\sqrt{Ct\epsilon}) \err_D \leq \err_{D'} + 4 \sqrt{\epsilon} \sqrt{\E_D y^4}+ 4\sqrt{Ct \epsilon} \E_D y^2$. 

Thus, if $\epsilon < 1/256Ct$, then, 
\[
\err_D \leq (1+ 32 \sqrt{Ct\epsilon}) \Paren{\err_{D'} + 4 \sqrt{Ct \epsilon} \Paren{\E_D y^2 + \sqrt{\E_D y^4}}}\mper
\]

This completes the proof of the first bound. 

Consider now, the special ``realizable'' case: $\err_{D'} = \E_{D'} (y' - \iprod{\ell,x'})^2 = 0$. Let $P$ be any deg $t$ polynomial such that $\E_{D'} \iprod{P'-{x'}^{\otimes t}}^2 = 0.$ Then, using certifiable hypercontractivity of $D$, 
\[
\E_{D} (y-\iprod{P,x^{\otimes t}})^4 = \E_{D} \iprod{P'-P,x^{\otimes t}}^4 \leq 2Ct(\E_{D} \iprod{P'-P,x^{\otimes t}}^2)^2.
\] 

Using the above estimate and rearranging \eqref{eq:CS-bound} with the fact that $\epsilon < 1/16C$ implies the second corollary in the lemma. 
\end{proof}
}

%% file: content/preliminaries.tex
\section{Sum of Squares proofs and Sum of Squares Optimization}
\label{sec:sospreliminaries}


In this section, we define pseudo-distributions and sum-of-squares proofs.
See the lecture notes~\citep{BarakS16} for more details and the appendix in~\citet{DBLP:journals/corr/MaSS16} for proofs of the propositions appearing here.

Let $x = (x_1, x_2, \ldots, x_n)$ be a tuple of $n$ indeterminates and let $\R[x]$ be the set of polynomials with real coefficients and indeterminates $x_1,\ldots,x_n$.
We say that a polynomial $p\in \R[x]$ is a \emph{sum-of-squares (sos)} if there are polynomials $q_1,\ldots,q_r$ such that $p=q_1^2 + \cdots + q_r^2$.

\ignore{
\begin{theorem}[\cite{BM:2002}] \label{generalizationbound}
	Let $\mathcal{D}$ be a distribution over $\mathcal{X} \times \mathcal{Y}$ and let $\mathcal{L} : \mathcal{Y}^\prime
	\times \mathcal{Y}$ (where $\mathcal{Y} \subseteq \mathcal{Y}^\prime \subseteq \mathbb{R}$) be a
	$b$-bounded loss function that is $L$-Lipschitz in its first argument.  Let
	$\mathcal{F} \subseteq (\mathcal{Y}^\prime)^\mathcal{X}$ and for any $f \in \mathcal{F}$, let $\mathcal{L}(f; \mathcal{D}) := \E_{(\textbf{x}, y)
	\sim \mathcal{D}}[\mathcal{L}(f(\textbf{x}), y)]$ and $\hat{\mathcal{L}}(f; S) := \frac{1}{n} \sum_{i = 1}^n
	\mathcal{L}(f(\textbf{x}_\textbf{i}), y_i)$, where $S = ((\textbf{x}_\textbf{1}, y_1), \ldots,  (\textbf{x}_\textbf{n}, y_n))  \sim
	\mathcal{D}^n$. Then for any $\delta > 0$, with probability at least $1 - \delta$
	(over the random sample draw for $S$), simultaneously for all $f \in
	\mathcal{F}$, the following is true:
	\[
		|\mathcal{L}(f; \mathcal{D}) - \hat{\mathcal{L}}(f; S)| \leq 4 \cdot L \cdot \mathcal{R}_\textbf{n}(\mathcal{F})
		+ 2\cdot b \cdot \sqrt{\frac{\log (1/\delta)}{2n}}
	\]
	where $\mathcal{R}_\textbf{n}(\mathcal{F})$ is the Rademacher complexity of the function class $\mathcal{F}$. 
\end{theorem}

For a linear concept class, the Rademacher complexity can be bounded as follows.

\begin{theorem}[\cite{KST:2008}] \label{rademachercomplexity}
	Let $\mathcal{X}$ be a subset of a Hilbert space equipped with inner product $\langle
	\cdot, \cdot \rangle$ such that for each $\textbf{x} \in \mathcal{X}$, $\langle \textbf{x}, \textbf{x}
	\rangle \leq X^2$, and let $\mathcal{W} = \{ \textbf{x} \mapsto \langle \textbf{x} , \textbf{w} \rangle
	~|~ \langle \textbf{w}, \textbf{w} \rangle \leq W^2 \}$ be a class of linear functions.
	Then it holds that
	\[
		\mathcal{R}_\textbf{n}(\mathcal{W}) \leq X \cdot W \cdot \sqrt{\frac{1}{n}}.
	\]
\end{theorem}

The following result is useful for bounding the Rademacher complexity of a smooth function of a concept class.

\begin{theorem}[\cite{BM:2002, LT:1991}]
\label{rademachercomplexity2}
	Let $\phi : \mathbb{R} \rightarrow \mathbb{R}$ be  $L_{\phi}$-Lipschitz
	and suppose that $\phi(0) = 0$. Let $\mathcal{Y} \subseteq \mathbb{R}$, and for a function $f \in \mathcal{Y}^{\mathcal{X}}$. 
	Finally, for $\mathcal{F} \subseteq \mathcal{Y}^{\mathcal{X}}$, let $\phi \circ \mathcal{F} = \{\phi \circ f \colon f \in \mathcal{F}\}$.
	It holds that $\mathcal{R}_\textbf{n}(\phi
	\circ \mathcal{F}) \leq 2 \cdot L_{\phi} \cdot \mathcal{R}_\textbf{n}(\mathcal{F})$.
\end{theorem}
}
\subsection{Pseudo-distributions}

Pseudo-distributions are generalizations of probability distributions.
We can represent a discrete (i.e., finitely supported) probability distribution over $\R^n$ by its probability mass function $D\from \R^n \to \R$ such that $D \geq 0$ and $\sum_{x \in \mathrm{supp}(D)} D(x) = 1$.
Similarly, we can describe a pseudo-distribution by its mass function.
Here, we relax the constraint $D\ge 0$ and only require that $D$ passes certain low-degree non-negativity tests.

Concretely, a \emph{level-$\ell$ pseudo-distribution} is a finitely-supported function $D:\R^n \rightarrow \R$ such that $\sum_{x} D(x) = 1$ and $\sum_{x} D(x) f(x)^2 \geq 0$ for every polynomial $f$ of degree at most $\ell/2$.
(Here, the summations are over the support of $D$.)
A straightforward polynomial-interpolation argument shows that every level-$\infty$-pseudo distribution satisfies $D\ge 0$ and is thus an actual probability distribution.
We define the \emph{pseudo-expectation} of a function $f$ on $\R^d$ with respect to a pseudo-distribution $D$, denoted $\pE_{D(x)} f(x)$, as
\begin{equation}
  \pE_{D(x)} f(x) = \sum_{x} D(x) f(x) \,\mper
\end{equation}
The degree-$\ell$ moment tensor of a pseudo-distribution $D$ is the tensor $\E_{D(x)} (1,x_1, x_2,\ldots, x_n)^{\otimes \ell}$.
In particular, the moment tensor has an entry corresponding to the pseudo-expectation of all monomials of degree at most $\ell$ in $x$.
The set of all degree-$\ell$ moment tensors of probability distribution is a convex set.
Similarly, the set of all degree-$\ell$ moment tensors of degree $d$ pseudo-distributions is also convex.
Key to the algorithmic utility of pseudo-distributions is the fact that while there can be no efficient separation oracle for the convex set of all degree-$\ell$ moment tensors of an actual probability distribution, there's a separation oracle running in time $n^{O(\ell)}$ for the convex set of the degree-$\ell$ moment tensors of all level-$\ell$ pseudodistributions.

\begin{fact}[\citep{MR939596-Shor87,parrilo2000structured,MR1748764-Nesterov00,MR1846160-Lasserre01}]
  \label{fact:sos-separation-efficient}
  For any $n,\ell \in \N$, the following set has a $n^{O(\ell)}$-time weak separation oracle (as defined in ~\citet{MR625550-Grotschel81}):
  \begin{equation}
    \Set{ \pE_{D(x)} (1,x_1, x_2, \ldots, x_n)^{\otimes d} \mid \text{ degree-d pseudo-distribution $D$ over $\R^n$}}\,\mper
  \end{equation}
\end{fact}
This fact, together with the equivalence of weak separation and optimization~\citep{MR625550-Grotschel81} allows us to efficiently optimize over pseudo-distributions (approximately)---this algorithm is referred to as the sum-of-squares algorithm.

The \emph{level-$\ell$ sum-of-squares algorithm} optimizes over the space of all level-$\ell$ pseudo-distributions that satisfy a given set of polynomial constraints---we formally define this next.

\begin{definition}[Constrained pseudo-distributions]
  Let $D$ be a level-$\ell$ pseudo-distribution over $\R^n$.
  Let $\cA = \{f_1\ge 0, f_2\ge 0, \ldots, f_m\ge 0\}$ be a system of $m$ polynomial inequality constraints.
  We say that \emph{$D$ satisfies the system of constraints $\cA$ at
    degree $r$}, denoted $D \sdtstile{r}{} \cA$, if for every
  $S\subseteq[m]$ and every sum-of-squares polynomial $h$ with $\deg h
  + \sum_{i\in S} \max\set{\deg f_i,r} \leq \ell$,
  \begin{displaymath}
    \pE_{D} h \cdot \prod _{i\in S}f_i  \ge 0\,.
  \end{displaymath}
  We write $D \sdtstile{}{} \cA$ (without specifying the degree) if $D \sdtstile{0}{} \cA$ holds.
  Furthermore, we say that $D\sdtstile{r}{}\cA$ holds \emph{approximately} if the above inequalities are satisfied up to an error of $2^{-n^\ell}\cdot \norm{h}\cdot\prod_{i\in S}\norm{f_i}$, where $\norm{\cdot}$ denotes the Euclidean norm\footnote{The choice of norm is not important here because the factor $2^{-n^\ell}$ swamps the effects of choosing another norm.} of the cofficients of a polynomial in the monomial basis.
\end{definition}

We remark that if $D$ is an actual (discrete) probability distribution, then we have  $D\sdtstile{}{}\cA$ if and only if $D$ is supported on solutions to the constraints $\cA$.

We say that a system $\cA$ of polynomial constraints is \emph{explicitly bounded} if it contains a constraint of the form $\{ \|x\|^2 \leq M\}$.
The following fact is a consequence of Fact~\ref{fact:sos-separation-efficient} and~\citet{MR625550-Grotschel81},

\begin{fact}[Efficient Optimization over Pseudo-distributions]
There exists an $(n+ m)^{O(\ell)} $-time algorithm that, given any explicitly bounded and satisfiable system\footnote{Here, we assume that the bitcomplexity of the constraints in $\cA$ is $(n+m)^{O(1)}$.} $\cA$ of $m$ polynomial constraints in $n$ variables, outputs a level-$\ell$ pseudo-distribution that satisfies $\cA$ approximately. 
\end{fact}

A property of pseudo-distributions that we will use frequently is the following:
\begin{fact}[H\"older's inequality] \label{fact:pseudo-Holders}
Let $f,g$ be SoS polynomials. 
Let $p,q$ be positive integers so that $1/p + 1/q = 1$. 
Then, for any pseudo-distribution $\tmu$ of degree $r \geq pq \cdot deg(f) \cdot deg(g)$, we have:
\[
(\pE_{\tmu} [f \cdot g])^{pq} \leq \pE[f^{p}]^{q} \cdot \pE[g^{q}]^{p}
\]
In particular, for all even integers $k \geq 2$, and polynomial $f$ with $deg(f) \cdot k \leq r$, 
$$(\pE_{\tmu}[ f])^k \leq \pE_{\tmu}[f^k].$$
\end{fact}

\subsection{Sum-of-squares proofs}

Let $f_1, f_2, \ldots, f_r$ and $g$ be multivariate polynomials in $x$.
A \emph{sum-of-squares proof} that the constraints $\{f_1 \geq 0,
\ldots, f_m \geq 0\}$ imply the constraint $\{g \geq 0\}$ consists of
(sum-of-squares) polynomials $(p_S)_{S \subseteq [m]}$ such that
\begin{equation}
g = \sum_{S \subseteq [m]} p_S \cdot \Pi_{i \in S} f_i
\mper
\end{equation}
We say that this proof has \emph{degree $\ell$} if for every set $S \subseteq [m]$, the polynomial $p_S \Pi_{i \in S} f_i$ has degree at most $\ell$.
If there is a degree $\ell$ SoS proof that $\{f_i \geq 0 \mid i \leq r\}$ implies $\{g \geq 0\}$, we write:
\begin{equation}
  \{f_i \geq 0 \mid i \leq r\} \sststile{\ell}{}\{g \geq 0\}
  \mper
\end{equation}

Sum-of-squares proofs satisfy the following inference rules.
For all polynomials $f,g\colon\R^n \to \R$ and for all functions $F\colon \R^n \to \R^m$, $G\colon \R^n \to \R^k$, $H\colon \R^{p} \to \R^n$ such that each of the coordinates of the outputs are polynomials of the inputs, we have:

\begin{align}
&\frac{\cA \sststile{\ell}{} \{f \geq 0, g \geq 0 \} } {\cA \sststile{\ell}{} \{f + g \geq 0\}}, \frac{\cA \sststile{\ell}{} \{f \geq 0\}, \cA \sststile{\ell'}{} \{g \geq 0\}} {\cA \sststile{\ell+\ell'}{} \{f \cdot g \geq 0\}} \tag{addition and multiplication}\\
&\frac{\cA \sststile{\ell}{} \cB, \cB \sststile{\ell'}{} C}{\cA \sststile{\ell \cdot \ell'}{} C}  \tag{transitivity}\\
&\frac{\{F \geq 0\} \sststile{\ell}{} \{G \geq 0\}}{\{F(H) \geq 0\} \sststile{\ell \cdot \deg(H)} {} \{G(H) \geq 0\}} \tag{substitution}\mper
\end{align}

Low-degree sum-of-squares proofs are sound and complete if we take low-level pseudo-distributions as models.

Concretely, sum-of-squares proofs allow us to deduce properties of pseudo-distributions that satisfy some constraints.

\begin{fact}[Soundness]
  \label{fact:sos-soundness}
  If $D \sdtstile{r}{} \cA$ for a level-$\ell$ pseudo-distribution $D$ and there exists a sum-of-squares proof $\cA \sststile{r'}{} \cB$, then $D \sdtstile{r\cdot r'+r'}{} \cB$.
\end{fact}

If the pseudo-distribution $D$ satisfies $\cA$ only approximately, soundness continues to hold if we require an upper bound on the bit-complexity of the sum-of-squares $\cA \sststile{r'}{} B$  (number of bits required to write down the proof).

In our applications, the bit complexity of all sum of squares proofs will be $n^{O(\ell)}$ (assuming that all numbers in the input have bit complexity $n^{O(1)}$).
This bound suffices in order to argue about pseudo-distributions that satisfy polynomial constraints approximately.

The following fact shows that every property of low-level pseudo-distributions can be derived by low-degree sum-of-squares proofs.

\begin{fact}[Completeness]
  \label{fact:sos-completeness}
  Suppose $d \geq r' \geq r$ and $\cA$ is a collection of polynomial constraints with degree at most $r$, and $\cA \vdash \{ \sum_{i = 1}^n x_i^2 \leq B\}$ for some finite $B$.

  Let $\{g \geq 0 \}$ be a polynomial constraint.
  If every degree-$d$ pseudo-distribution that satisfies $D \sdtstile{r}{} \cA$ also satisfies $D \sdtstile{r'}{} \{g \geq 0 \}$, then for every $\epsilon > 0$, there is a sum-of-squares proof $\cA \sststile{d}{} \{g \geq - \epsilon \}$.
\end{fact}

We will use the following standard sum-of-squares inequalities:

\begin{fact}[SoS H\"older's Inequality]
Let $f_1, f_2, \ldots, f_n$ and $g_1, g_2, \ldots, g_n$ be SoS polynomials over $\R^d$. Let $p, q$ be integers such that $1/p + 1/q = 1$. Then, 
\[
\sststile{pq}{f_1, \ldots, f_n,g_1, \ldots, g_n} \Set{ \Paren{\frac{1}{n} \sum_{i} f_i g_i }^{pq} \leq \Paren{\frac{1}{n} \sum_{i=1}^n f_i^p}^q \Paren{\frac{1}{n} \sum_{i=1}^n g_i^q}^p }
\]

\end{fact}
\begin{fact}
For any $a_1, a_2,\ldots,a_n$, 
\[
\sststile{k}{a_1, a_2, \ldots, a_n} \Set{ (\sum_i a_i)^k \leq n^k \Paren{\sum_i a_i^k} } 
\]
\end{fact}

%% file: content/algorithm.tex
\section{Algorithm}
In this section, we present and analyze our robust regression algorithms. 
We begin by setting some notation that we will use throughout this section:
\begin{enumerate}
\item $\cD$ denotes the uncorrupted distribution on $\R^d \times \R$. In general, calligraphic letters will denote distributions on example-label pairs. $D = \cD_x$ will denote the marginal distribution on $x$. 

\item We will write $X= ((x_1,y_1), (x_2,y_2), \ldots, (x_n,y_n))$ to denote the uncorrupted input sample of size $n$ drawn according to $\cD$. For some bound $B$ on the \emph{bit-complexity} of linear functions, we will write $\opt(\cD)$ for the optimum least squares error of any linear function of bit complexity $B$ on $\cD$. Recall that the bit complexity of a linear function is the number of bits required to write down all of its coefficients.  

\item We will write $\widehat{\cD}$ for the uniform distribution on the sample $X$. $\widehat{D} = \widehat{\cD}_x$ will denote the marginal distribution on $x$. Note that our algorithm does not get direct access to $\cD$ or $\widehat{\cD}$. We will write $\opt(\widehat{\cD})$ for the optimum least squares error of any linear function of bit complexity $B$ on $\widehat{\cD}$.

\item We will write $U = ((u_1, v_1), (u_2, v_2), \ldots, (u_n, v_n))$ to denote an $\eta$-corruption of $X$, i.e., $U$ is obtained by changing $\eta$ fraction of the example-label pairs. Our algorithm gets access to $U$. 
\item For $\ell \in \R^d$, and $M > 0$, let $\ell_M:\R^d \to \R$ denote the truncated linear function defined as follows: 
$$\ell_M (x) = \begin{cases} \iprod{\ell,x} & \text{ if } |\iprod{\ell,x}| \leq M \\
       \sign(\iprod{\ell,x}) \cdot M & \text{ otherwise.}
       \end{cases}.$$

\end{enumerate}


\subsection{Robust Least Squares Regression} \label{sec:robust-L2-regression-algo}
In this section, we present our Robust Least Squares Regression algorithm. The main goal of this section is to establish the following result.

\newcommand{\clip}{\mathsf{clip}}
\begin{theorem} \label{thm:analysis-L2-linear-regression}
Let $\cD$ be a distribution on $\R^d \times [-M,M]$ for some positive real $M$ such the marginal on $\R^d$ is $(C,k)$-certifiably hypercontractive distribution. Let $\opt_B(\cD) = \min_{\ell} \E_{\cD}[ (y - \iprod{\ell,x})^2]$ where the minimum is over all $\ell \in \R^d$ of bit complexity $B$. Let $\ell^{*}$ be any such minimizer. 

Fix any even $k \geq 4$ and any $\epsilon > 0$. Let $X$ be an i.i.d. sample from $\cD$ of size $n \geq n_0 = \poly(d^k, B, M, 1/\epsilon)$. 
Then, with probability at least $1-\epsilon$ over the draw of the sample $X$, given any $\eta$-corruption $U$ of $X$ and $\eta$ as input, there is a polynomial time algorithm (Algorithm \ref{alg:robust-regression-program}) that outputs a $\ell \in \R^d$ such that for $C = C(k/2)$, 
\[
\err_{\cD}(\ell_M) < (1 + O(C)\eta^{1-2/k}) \opt_B(\cD) + O(C) \eta^{1-2/k} \Paren{\E_{\cD} (y-\iprod{\ell^{*},x})^k}^{2/k} + \epsilon
.
\] 
\end{theorem}

By an entirely analogous argument, we also get a similar guarantee for outlier-robust polynomial regression. We defer the details to Section \ref{sec:app-moved}.

We need the boundedness assumption on the labels $y$ (that they lie in $[-M,M]$) and the bounded bit-complexity assumption on the linear functions ($B$) mainly to obtain generalization bounds for linear regression as are often used even for regression without corruptions. Further note that specializing the above to the case $k=4$ gives Theorem \ref{th:intro4}. 

Following the outline described in the introduction, we first define a set of polynomial inequalities which will be useful in our algorithm: Let $\eta > 0$ be a parameter and consider the following system of polynomial inequalities in variables $w \in \R^n$, $\ell \in \R^d$, $x'_1,\ldots,x'_n \in \R^d$: 

\begin{equation}
  \cP_{U,\eta} = 
  \left \{
    \begin{aligned}
      & \textstyle\sum_{i=1}^n w_i
      = (1-\eta) \cdot n & &\\
      & w_i^2
      =w_i 
      &\forall i\in [n]. &\\
      & w_i \cdot (u_i - x'_i)
       = 0
      &\forall i\in [n]. &\\
      & w_i \cdot (v_i - y'_i)
      = 0
      &\forall i\in [n]. &
    \end{aligned}
  \right \} \label{eq:polynomial-constraints-w}
\end{equation}

\ignore{
\begin{equation}
  \cA_{U,\eta} = 
  \left \{ (w,\ell, X' = (x'_1,\ldots,x'_n)) \} \in \R^n \times \R^d \times (\R^d)^n \,:\, 
    \begin{aligned}
      &&
      \textstyle\sum_{i=1}^n w_i
      &= (1-\eta) \cdot n\\
      &\forall i\in [n].
      & w_i^2
      & =w_i \\
      &\forall i\in [n].
      & w_i \cdot (x_i - x'_i)
      & = 0\\
      &\forall i\in [n].
      & w_i \cdot (y_i - y'_i)
      & = 0
    \end{aligned}
  \right \} \label{eq:polynomial-constraints-w}
\end{equation}
}

Observe that this system is feasible: use $w_i = 1$ if $(x_i,y_i) = (u_i, v_i)$ and $0$ otherwise (i.e., $w_i = 1$ if and only if the $i$'th example was corrupted) and taking $(x_i', y_i') = (x_i,y_i)$ for all $i \in [n]$.

We are now ready to describe our algorithm for robust L2 regression. 

                                  
\begin{mdframed}
  \begin{algorithms}[Algorithm for Robust L2 Linear Regression via sum-of-squares]
    \label{alg:robust-regression-program}\mbox{}
    \begin{description}
    \item[Given:]
    \begin{itemize}
    \item $\eta$: A bound on the fraction of adversarial corruptions. 
    \item $U$: An $\eta$-corruption of a labeled sample $X$ of size $n$ sampled from a $(C,k)$-certifiably hypercontractive distribution $\cD$.

    \end{itemize}
    \item[Operation:]\mbox{}
      \begin{enumerate}
      \item 
        Find a level-$k$ pseudo-distribution $\tmu$ that satisfies $\cP_{U,\eta}$ and minimizes $\pE_{\tmu}\left[\Paren{\frac{1}{n} \sum_{i = 1}^n (y_i' - \iprod{\ell,x_i'})^2}^{k/2}\right]$. Let $\widehat{\opt}_{SOS}$ be a positive real number so that $\widehat{\opt}_{SOS}^{k/2}$ is this minimum value. 
      \item Output $\widehat{\ell} = \pE_{\tmu}\ell$.
  \ignore{    \item Output $f:\R^d \rightarrow \R$ given by 
      \[
      f(x) = \begin{cases} \iprod{\widehat{\ell},x} &\text{ if } |\iprod{\widehat{\ell},x}| \leq M\\
                           \sign(\iprod{\widehat{\ell},x}) \cdot M &\text{ otherwise.}  \end{cases} 
                           \]}
      \end{enumerate}
    \end{description}    
  \end{algorithms}
\end{mdframed}

\subsection{Analysis of the Algorithm}
We now analyze the algorithm and prove Theorem \ref{thm:analysis-L2-linear-regression}. The analysis can be broken into two modular steps: (1) Bounding the \emph{optimization error} (roughly translates to bounding the empirical error) and (2) Bounding the \emph{generalization error}. Concretely, we break down the analysis into the following two steps. Let $\widehat{\opt}_k = \left((1/n) \sum_{i=1}^n (y_i - \iprod{\ell^*,x})^k\right)^{2/k}$ and $\opt_k(\cD) = \E_{(x,y) \cD}[ (y - \iprod{\ell^*,x})^k]^{2/k}$. 

\begin{lemma}[Bounding the optimization error]\label{lm:opterror}
Under the assumptions of Theorem \ref{thm:analysis-L2-linear-regression} (and following the above notations), with probability at least $1-\epsilon$, 
$$\err_{\widehat{\cD}}(\widehat{\ell}) \leq (1 + C(k/2) \eta^{1-2/k} ) \cdot \widehat{\opt}_{SOS} + O(C(k/2)) \cdot \eta^{1-2/k} \cdot \widehat{\opt}_k.$$
\end{lemma}

\begin{lemma}[Bounding the generalization error]\label{lm:generror}
Under the assumptions of Theorem \ref{thm:analysis-L2-linear-regression}, with probability at least $1-\epsilon$, the following hold:
\begin{enumerate}
\item $\widehat{\opt}_{SOS}  \leq \opt(\cD) + \epsilon$.

\item $\err_{\cD}(\widehat{\ell}_M)) \leq \err_{\widehat{\cD}}(\widehat{\ell}) + \epsilon$.
\end{enumerate}
\end{lemma}

Ideally, we would liked to also have $\widehat{\opt}_k \leq  \opt_k(\cD) + \epsilon$. Given such an inequality, Theorem \ref{thm:analysis-L2-linear-regression} would follow immediately from the above two lemmas. A small technical issue is that we cannot prove such an inequality as we don't have good control on the moments of $(y - \iprod{\ell^*,x})^k$. However, we can exploit the robust setting to get around this issue by essentially truncating large values - since the distribution with truncated values will be close in statistical distance to the actual distribution. We remark that the proof of Lemma \ref{lm:generror} follows standard generalization arguments for the most part. 

We defer the proofs of the above lemmas and proceed to finish analyzing our algorithm. With Lemma \ref{lm:opterror}, \ref{lm:generror} in hand, we are now ready to prove our main theorem. We just need the following lemma to get around bounding $\widehat{\opt}_k$. 

\begin{lemma} \label{lem:little-boundedness-fact}
For every distribution $\cD$ on $\R^d \times \R$ such that $v = E_{\cD} (y - \iprod{\ell^{*},x})^k < \infty$, there exists a distribution $\cF$ such that $\|\cD-\cF\|_{TV} < \eta$ and $(y - \iprod{\ell^{*},x})^k$ is bounded absolutely bounded in the support of $\cF$  by $v/\eta$.
\end{lemma}
\begin{proof}
Set $\cF = \cD \mid \left((y-\iprod{\ell^{*},x})^k \leq v/\eta\right)$. Then, by definition $\cF$ satisfies the property that $(y - \iprod{\ell^{*},x})^k$ is bounded by $v/\eta$ in the support of $\cF$. Further, by Markov's inequality, the probability of the event we conditioned on is at least $1-\eta$. This completes the proof.
\end{proof}

Observe that an $\eta$ corrupted sample from $\cD$ can be thought of as an $2\eta$ corrupted sample from $\cF$. Since $(y-\iprod{\ell^{*},x})^k$ is bounded in $\cF$, it allows us to use Hoeffding bound for concentration to show that the empirical expectation of $(y-\iprod{\ell^{*},x})^k$ converges to its expectation under $\cD$.


\begin{proof}[Proof of Theorem \ref{thm:analysis-L2-linear-regression}] Let $X$ be an i.i.d. sample from $\cD$ of size $n$ and $\hat{\cD}$ be the uniform distribution on $X$.
Let $v = \E_{\cD} (y- \iprod{\ell^{*},x})^k$. Without loss of generality, by using Fact \ref{lem:little-boundedness-fact}, we can assume that $(y- \iprod{\ell^{*},x})^k$ is bounded above by $v/\eta$ in $\cD$. Using Hoeffding's inequality, if $n \geq v\log{(1/\delta)}/\eta \epsilon^2$, then with probability at least $1-\delta$, $\widehat{\opt}_k = \E_{\hat{\cD}} [(y - \iprod{\ell^{*},x})^k] \leq \E_{\cD} (y - \iprod{\ell^{*},x})^k + \epsilon = \opt_k + \epsilon.$

Therefore, by Lemmas \ref{lm:opterror}, \ref{lm:generror}, and the above observation, we get that with probability at least $1 - O(\epsilon)$, 

$\err_{\cD}(\ell_M) \leq (1 + O(C) \eta^{1-2/k}) \cdot \opt(\cD) + O(C) \eta^{1-2/k} \cdot \opt_k + O( C \epsilon).$
The theorem now follows by choosing $\epsilon$ to be a sufficiently small constant times the parameter desired. 
\end{proof}
\ignore{
Next, if $n \geq \tilde{\Theta}(d^{k/2} \log{(d)}^{k/2})$, then by Fact \ref{fact:hypercontractivity-preserved-under-sampling}, with probability at least $1-1/d$, uniform distribution $\hat{\cD}$ on $X$ is $(C,k)$-hypercontractive. Thus, Lemma \ref{lem:analysis-full-L2} applies to the uniform distribution on $X$ and combined with Lemma \ref{lem:error-of-output-is-small}, we have that $\hat{\ell}$ constructed by Algorithm \ref{alg:robust-regression-program} when run on input $U$ for any $\eta$ corruption of $X$ satisfies: 

\[
\err_{\hat{\cD}}(\hat{\ell}) < (1+O(C)\eta^{1-2/k})  \opt(\cD) + \epsilon + O(C) \eta^{1-2/k} \Paren{\frac{1}{n} \sum_{i = 1}^n (y_i - \iprod{\ell^{*},x})^k}^{2/k}  \mper
\]

Finally, consider the truncation $f = f_{\hat{\ell}}$ defined by 
\[
f (x) = \begin{cases} \iprod{\hat{\ell},x} & \text{ if } |\iprod{\hat{\ell},x}| \leq M \\
       \sign(\iprod{\hat{\ell},x}) \cdot M & \text{ otherwise.}
       \end{cases}
\]

Then, the square loss $(y-f(x))$ for any $(x,y)$ in the support of $\cD$ takes values in $[-2M,2M]$. Further, the bit complexity of $f$ is bounded above by $B + \log{(M)}$ since the bit complexity of  $\ell$ is at most $B$ giving a bound of $M^B$ on the size of the class of all such truncated linear functions.  Thus, we can apply Fact \ref{fact:standard-generalization-bound} to this class and obtain that if $n \geq O(\log{(1/\delta)} M^2 B \log{(M)}/\epsilon^2)$ then the error of $f_{\hat{\ell}}$ on $\cD$ is at most $\err_{\hat{\cD}}(\hat{\ell}) + \epsilon$ with probability at least $1- \delta$. This completes the proof.








}

%% file: content/algorithm-proofs.tex
\subsubsection{Bounding the Optimization Error} \label{sec:proofoftheorem}
We now prove Lemma~\ref{lm:opterror}. While the proof can appear technical, it's essentially a line-by-line translation of the robust certifiability Lemma~\ref{lem:identifiability-least-squares-linear}.

\paragraph{Proof Outline} The rough idea is to exploit the following abstract property of pseudo-distributions: If a collection of polynomial inequalities $\cP = \{p_i(z) \geq 0, i \in [r]\}$ SOS-imply another polynomial inequality $q(z) \geq 0$, then any pseudo-distribution $\tmu$ of appropriately high degree (depending on the degree of the SOS proof) that satisfies the inequalities in $\cP$ also satisfies $q$, that is $\pE_{\tmu}[q] \geq 0$. Further, the SoS algorithm allows us to compute pseudo-distributions satisfying a set of polynomial inequalities efficiently. 

Now, let $(w,\ell,X')$ satisfy the inequalities $\cP_{U,\eta}$. Then, by Lemma \ref{lem:identifiability-least-squares-linear}, applied to $\widehat{\cD}$ and the uniform distribution on $X'$, we get 

$$\err_{\widehat{\cD}}(\ell) \leq (1 + c C \eta^{1-2/k}) \left((1/n) \sum_{i=1}^n  (y_i' - \iprod{\ell,x_i'})^2\right) + c C \eta^{1-2/k} \cdot \widehat{\opt}_k,$$
for some universal constant $c > 0$. 

To view the above inequality as a polynomial inequality in variables $w,\ell,X'$, we rephrase it as follows. For brevity, let $\err(w,\ell,X') = (1/n) \sum_{i=1}^n (y_i' - \iprod{\ell,x_i'})^2$. Then, 
$$ \Paren{\err_{\hat{\cD}}(\ell) - \err(w,\ell,X')}^{k/2}  \leq \eta^{k/2-1} \cdot 2^{\Theta(k)} C^{k} \err(w,\ell,X')^{k/2} + \eta^{k/2-1} \cdot 2^{\Theta(k)} C^{k} \cdot \widehat{\opt}_k^{k/2} .$$

We show that the above version of the robust certifiability lemma has a SOS proof; that is, viewing the above inequality as a polynomial inequality in variables $w,\ell,X'$, this inequality has a SOS proof starting from the polynomial inequalities $\cP_{U,\eta}$. Thus, by the property of pseud-densities at the beginning of this sketch, a pseudo-density $\tmu$ as in our algorithm satisfies an analogue of the above inequality which after some elementary simplifications gives us a bound of the form 
$$\pE_{\tmu}[ \err_{\widehat{\cD}}(\ell)] \leq (1 + c C \eta^{1-2/k}) \cdot \widehat{\opt}_{SOS} + c C \eta^{1-2/k} \widehat{\opt}_k.$$

As it stands, the above inequality is not very useful for us as it does not tell us which $\ell$ to choose. However, for any degree at most $k/2$ polynomial $p$, we also have that $(\pE_{\tmu}[ p(w,\ell)])^2 \leq \pE[p(w,\ell)^2] $ (see Fact \ref{fact:pseudo-Holders}). Applying this to each $(y_i - \iprod{\ell,x_i})$, we get that
$$\err_{\widehat{\cD}} (\pE_{\tmu}[\ell]) \leq \pE_{\tmu}[ \err_{\widehat{\cD}}(\ell)]  \leq (1 + c C \eta^{1-2/k}) \cdot \widehat{\opt}_{SOS} + c C \eta^{1-2/k} \widehat{\opt}_k,$$
proving the claim. 

We next formalize the above approach starting with a SOS proof of Lemma \ref{lem:identifiability-least-squares-linear}. We defer the proof of the lemma to Section \ref{sec:soscertification}.



\newcommand{\U}{\mathcal{U}}
\begin{lemma}[SoS Proof of Robust Certifiability of Regression Hypothesis]

Let $X$ be a collection of $n$ labeled examples in $\R^d\times \R$ such that $\widehat{\cD}$, the uniform distribution on  $x_1, x_2, \ldots, x_n$ is $k$-certifiably hypercontractive and all the labels $y_1, y_2, \ldots, y_n$ are bounded in $[-M,M]$. Let $U$ be an $\eta$-corruption of $X$.

Let $(w,\ell,X')$ satisfy the set of system of polynomial equations $\cP_{U,\eta}$. Let $\err_{\hat{\cD}}(\ell)$ be the quadratic polynomial $\E_{(x,y) \sim \hat{\cD}} (y - \iprod{\ell,x})^2$ in vector valued variable $\ell$. Let $\err(w,\ell,X')$ be the polynomial $\frac{1}{n} \sum_{i \leq n} (y'_i - \iprod{\ell, x'_i})^2$ in vector valued variables $w,\ell,x_1',\ldots,x_n'$.

Then, for any $\ell^{*} \in \R^d$ of bit complexity at most $B < \poly(n,d^k)$, $C = C(k/2)$ and any $\eta$ such that $ 100 C \eta^{1-2/k} < 0.9$, 
 




\begin{multline}
\cA_{U,\eta} \sststile{k}{\ell} \Paren{\err_{\hat{\cD}}(\ell) - \err(w,\ell,X')}^{k/2}  \leq \eta^{k/2-1} \cdot 2^{\Theta(k)} C^{k} \err(w,\ell,X')^{k/2} \\+ \eta^{k/2-1} \cdot 2^{\Theta(k)} C^{k} \Paren{ \frac{1}{n} \sum_{i =1}^n (y_i-\iprod{\ell^{*},x_i})^{k}}  \mper
\end{multline}

Moreover, the bit complexity of the proof is polynomial in $n$ and $d^k$.

\label{lem:identifiability-least-squares-linear-sos}
\end{lemma}

We also need the following lemma (that follows from appropriate matrix concentration results) from \cite{DBLP:journals/corr/abs-1711-11581} stating that the uniform distribution on a sufficiently large set of i.i.d samples from a hypercontractive distribution also satisfy hypercontractivity. This allows us to argue that the uncorrupted empirical distribution $\widehat{\cD}$ is also hypercontractive when $\cD$ is. 



\begin{lemma}[Lemma 5.5 of \cite{DBLP:journals/corr/abs-1711-11581}]\label{fact:hypercontractivity-preserved-under-sampling}
Let $D$ be a $(C,k)$-certifiably hypercontractive distribution on $\R^d$. Let $X$ be an i.i.d. sample from $D$ of size $n \geq \Omega( (d^{k/2} \log{(d/\delta)})^{k/2})$. Then, with probability at least $1-\delta$ over the draw of $X$, the uniform distribution $D$ over $X$ is $(2C,k)$-certifiably hypercontractive.
\end{lemma}

We can now prove Lemma \ref{lm:opterror}.
\begin{proof}[Proof of Lemma \ref{lm:opterror}]
If $n \geq \tilde{\Theta}(d^{k/2} \log{(d/\epsilon)}^{k/2})$, then by Lemma \ref{fact:hypercontractivity-preserved-under-sampling}, with probability at least $1-\epsilon$, the uniform distribution $\widehat{\cD}$ on $X$ is $(2 C,k)$-hypercontractive. 

Since $\tmu$ is a pseudo-distribution of level $k$, combining Fact \ref{fact:sos-soundness} and Lemma \ref{lem:identifiability-least-squares-linear-sos}, we must have for $C = C(k/2)$,
\begin{equation} \label{eq:proof-to-pseudo-dist}
\pE_{\tmu} \Paren{\err_{\hat{\cD}}(\ell) - \err(w,\ell,X')}^{k/2} \leq O(C^{k/2} \eta^{k/2-1}) \cdot \pE_{\tmu} \err(w,\ell,X')^{k/2} + O(C^{k/2}\eta^{k/2-1}) \cdot \Paren{\E_{\hat{\cD}} (y-\iprod{\ell^{*},x})^k} .
\end{equation}

Taking $2/k$th powers of both sides of the above equation and recalling the definition of $\widehat{\opt}_{SOS}, \widehat{\opt}_k$, we get
\[
(\pE_{\tmu} \err_{\hat{\cD}}(\ell)- \err_{\cD'}(\ell))^{k/2})^{2/k} \leq O(C) \eta^{1-2/k} \cdot \widehat{\opt}_{SOS} + O(C) \eta^{1-2/k} \widehat{\opt}_k^{2/k}\mper
\]

Now, by Fact \ref{fact:pseudo-Holders},
$(\pE_{\tmu} [\err_{\hat{\cD}} (\ell) -\err(w,\ell,X')])^{k/2} \leq \pE_{\tmu}[ (\err_{\hat{\cD}}(\ell) - \err(w,\ell,X'))]^{k/2}$ and thus, 

\[
\pE_{\tmu} \err_{\hat{\cD}} (\ell) \leq (1 + O(C) \eta^{1-2/k}) \cdot \widehat{\opt}_{SOS} + O(C) \eta^{1-2/k} \widehat{\opt}_k \mper
\]

Finally, by another application of Fact \ref{fact:pseudo-Holders}, we have that for every $i$, $(y_i - \iprod{x_i,\pE_{\tmu}[\ell]})^2 \leq \pE_{\tmu}[(y_i \iprod{x_i,\ell})^2]$; in particular,  $\err_{\widehat{\cD}}(\pE_{\tmu}[\ell]) \leq \pE_{\tmu} \err_{\hat{\cD}} (\ell)$. Thus, we have
$$\err_{\widehat{\cD}}(\pE_{\tmu}[\ell]) \leq (1 + O(C) \eta^{1-2/k}) \cdot \widehat{\opt}_{SOS} + O(C) \eta^{1-2/k} \widehat{\opt}_k,$$
proving the lemma. 
\end{proof}

\subsubsection{Proof of Lemma \ref{lem:identifiability-least-squares-linear-sos}}\label{sec:soscertification}
Here we prove Lemma \ref{lem:identifiability-least-squares-linear-sos}. The proof is similar in spirit to that of Lemma \ref{lem:identifiability-least-squares-linear} but we need to adapt the various steps to a form suitable for SOS proof system.
\begin{proof}[Proof of Lemma \ref{lem:identifiability-least-squares-linear-sos}]

For brevity, we write $\err_{\cD}$ for $\err_{\cD}(\ell)$.

Let $w' \in \zo^n$ be given by $w'_i = w_i$ iff $i$th sample is uncorrupted in $U$ and $0$ otherwise.
Then, observe that $\sum_i w_i' = s$ for $s \geq (1-2\eta)n.$ 

Then, 
\[
\sststile{w'}{2} \Set{ \frac{1}{n} \sum_i (1-w'_i)^2 \leq 2 \eta}\mper
\]

Let $\err_{w'}(\ell) = \frac{1}{n} \sum_{i = 1}^n w_i' (v_i - \iprod{\ell,u_i})^2$.
We have:
\[
\sststile{w,\ell}{4} \err_{\cD}(\ell) = \frac{1}{n} \sum_{i = 1}^n w'_i (y_i-\iprod{\ell,x_i})^2 + \sum_{i = 1}^n (1-w'_i) \cdot (y_i-\iprod{\ell,x_i})^2
\]

On the other hand, we also have:
\[
\sststile{w,\ell}{4} \frac{1}{n} \sum_{i = 1}^n w'_i (y_i-\iprod{\ell,x_i})^2 \leq \sum_{i =1}^n (y_i' - \iprod{\ell,x'_i})^2 = \err_{\cD'}(\ell)\mper
\]

Combining the above and using the sum-of-squares vesion of the H\"older's inequality, we have:

\begin{align}
\sststile{w,\ell}{k} \Paren{\err_{\cD}(\ell) - \err_{\cD'}(\ell)}^{k/2} &= \Paren{\frac{1}{n} \sum_{i = 1}^n (1-w'_i) \cdot (y_i -\iprod{\ell,x_i})^2}^{k/2}\notag\\
&\leq \Paren{\frac{1}{n} \sum_{i = 1}^n (1-w'_i)}^{k/2-1} \Paren{ \frac{1}{n} \sum_{i =1}^n (y_i-\iprod{\ell,x_i})^{k}} \notag\\
&\leq 2^{k/2-1}\eta^{k/2-1} \Paren{ \frac{1}{n} \sum_{i =1}^n (y_i-\iprod{\ell,x_i})^{k}}  \label{eq:CS-bound}\mper
\end{align}

Next, using the sum-of-squares inequality $(a+b)^k \leq 2^k a^k + 2^k b^k$, we have:
\begin{equation}\label{eq:minkowski}
\sststile{\ell}{k} \Set{\Paren{ \frac{1}{n} \sum_{i =1}^n (y_i-\iprod{\ell,x_i})^{k}} \leq 2^{k} \Paren{ \frac{1}{n} \sum_{i =1}^n (y_i-\iprod{\ell^{*},x_i})^{k}} + 2^{k} \Paren{ \frac{1}{n} \sum_{i =1}^n \iprod{\ell-\ell^{*},x_i}^{k}} } 
\end{equation}

By certifiable hypercontractivity of $\cD_x = D$, we have:
\[
\sststile{\ell}{k} \Set{\Paren{ \frac{1}{n} \sum_{i =1}^n \iprod{\ell-\ell^{*},x_i}^{k}} \leq C(k)^{k/2} \Paren{ \frac{1}{n} \sum_{i =1}^n \iprod{\ell-\ell^{*},x_i}^2}^{k/2}} 
\]

Again, by using the sum-of-squares inequality $(a+b)^k \leq 2^k a^k + 2^k b^k$, we have:
\[
\sststile{\ell}{k} \Set{\Paren{ \frac{1}{n} \sum_{i =1}^n \iprod{\ell-\ell^{*},x_i}^2}^{k/2} \leq 2^{k/2}\Paren{ \frac{1}{n} \sum_{i =1}^n (y_i-\iprod{\ell,x_i}^2}^{k/2} + 2^{k/2}\Paren{ \frac{1}{n} \sum_{i =1}^n (y_i-\iprod{\ell^{*},x_i}^2}^{k/2}  } 
\]
Finally, using the sum-of-squares version of H\"older's inequality again, we have:
\[
\sststile{\ell}{k} \Set{\Paren{ \frac{1}{n} \sum_{i =1}^n (y_i-\iprod{\ell^{*},x_i})^2}^{k/2} \leq  \frac{1}{n} \sum_{i =1}^n (y_i-\iprod{\ell^{*},x_i})^k   } 
\]
Combining the above with \eqref{eq:minkowski}, we have:
\begin{multline}
\sststile{\ell}{k}\\ \Set{\Paren{ \frac{1}{n} \sum_{i =1}^n (y_i-\iprod{\ell,x_i})^{k}} \leq O(C(k/2))^{k} \Paren{ \frac{1}{n} \sum_{i =1}^n (y_i-\iprod{\ell^{*},x_i})^{k}} + O(C(k/2))^{k}\Paren{ \frac{1}{n} \sum_{i =1}^n (y-\iprod{\ell,x_i})^2}^{k/2}} 
\end{multline}

Thus, together with \eqref{eq:CS-bound},  we have:
\begin{multline}
\sststile{\ell}{k} \Paren{\err_{\cD}(\ell) - \err_{\cD'}(\ell)}^{k/2}  \leq \eta^{k/2-1} \cdot O(C(k/2))^{k} (\err_{\cD}(\ell))^{k/2} \\+ \eta^{k/2-1} \cdot O(C(k/2))^{k} \Paren{ \frac{1}{n} \sum_{i =1}^n (y_i-\iprod{\ell^{*},x_i})^{k}}  
\end{multline}
Using the sum of squares inequality $\delta^k a^k \leq (2\delta)^k (a-b)^k + (2\delta)^k b^k$ for any $a,b$ and even $k$, and applying it with $a = \err_{\cD}(\ell)$, $b = \err_{\cD'}(\ell)$ and $\delta = \eta^{k/2-1} \cdot O(C(k/2))^{k}$ and rearranging, we have:

\begin{multline}
\sststile{\ell}{k}  (1-\delta) \Paren{\err_{\cD}(\ell) - \err_{\cD'}(\ell)}^{k/2}  \leq \eta^{k/2-1} \cdot O(C(k/2))^{k} (\err_{D'}(\ell))^{k/2} \\+ \eta^{k/2-1} \cdot O(C(k/2))^{k} \Paren{ \frac{1}{n} \sum_{i =1}^n (y_i-\iprod{\ell^{*},x_i})^{k}} 
\end{multline}

For $\delta < 0.9$, this implies:

\begin{multline}
\sststile{\ell}{k}  \Paren{\err_{\cD}(\ell) - \err_{\cD'}(\ell)}^{k/2}  \leq \eta^{k/2-1} \cdot O(C(k/2))^{k} (\err_{\cD'}(\ell))^{k/2} \\+ \eta^{k/2-1} \cdot O(C(k/2))^{k} \Paren{ \frac{1}{n} \sum_{i =1}^n (y_i-\iprod{\ell^{*},x_i})^{k}}  
\end{multline}
This completes the proof.
\end{proof}

\subsubsection{Bounding the Generalization Error}
In this section we prove Lemma \ref{lm:generror}. The lemma follows from standard concentration inequalities combined with standard generalization bounds for linear regression.

\begin{proof}[Proof of Lemma \ref{lm:generror}(1)]
Let $\ell^{*}$ be a linear function of bit complexity at most $B$ that achieves the optimum least squares regression error on $\cD$. We will first show that $\widehat{\opt}_{SOS} \leq \err_{\widehat{\cD}}(\ell^*)$ by exhibiting a feasible pseudo-density. To see this, consider the point-mass density, $\tmu$, supported on the following point: $(w,\ell^*,X')$ where $w_i = 1$ if $(x_i,y_i) = (u_i, v_i)$ and $0$ otherwise (i.e., $w_i = 1$ if and only if the $i$'th example was uncorrupted) and $(x_i', y_i') = (x_i,y_i)$ for all $i \in [n]$. Clearly, $\tmu$ is a feasible solution to the optimization progam \ref{alg:robust-regression-program} and $\pE_{\tmu}[\err(w,\ell,X)^{k/2}] = \left((1/n) (\sum_{i=1}^n (y_i - \iprod{\ell^*,x_i})^2)\right)^{k/2} = \err_{\widehat{\cD}}(\ell^*)^{k/2}$. It follows that $\widehat{\opt}_{SOS} \leq  \err_{\widehat{\cD}}(\ell^*)$. 

We next argue that $\err_{\widehat{\cD}}(\ell^*)$ is close to $\err_{\cD}(\ell^*)$ for $n$ sufficiently big. Let $Z$ be the random variable $(y - \iprod{\ell^*,x})^2$ for $(x,y) \sim \cD$. Note that $\err_{\widehat{\cD}}(\ell^*)$ is the average of $n$ independent draws of the random variable $Z$. Also note that $\E[Z] = \opt(\cD)$. We will next bound the variance of $Z$. We have, for $(x,y) \sim \cD$,
\begin{equation*}
\E[Z^2] = \E[(y - \iprod{\ell^*,x})^4] \leq 2 \E[y^4] + 2 \E[\iprod{\ell^*,x}^4] \leq 2 M^4 + 2 C^2 (\E[\iprod{\ell^*,x}^2])^2, 
\end{equation*}
where the last inequality follows by hypercontractivity.
Now, $\E[\iprod{\ell^*,x}^2] \leq 2 \E[(y - \iprod{\ell^*,x})^2] + 2 \E[y^2] \leq 2 \opt(\cD) + 2 M^2 \leq 4 M^2$ as $\opt(\cD) \leq M^2$ (the $0$ function achieves this error). Combining the above we get that $\E[Z^2] = O(M^4)$. 

Thus, for some $n_0 = O(1/ \epsilon^3) (M^4)$, if we take $n \geq n_0$ independent samples $Z_1, Z_2, \ldots,Z_n$ of $Z$, then $\Pr[ |\frac{1}{n} \sum_{i = 1}^n Z_i - \E[Z]| \geq \epsilon] \leq \epsilon$. Thus, with probability at least $1-\epsilon$, $\err_{\widehat{\cD}}(\ell^*) \leq \opt(\cD) + \epsilon$. The claim now follows.
\end{proof}

Part (2) of Lemma \ref{lm:generror} follows from standard generalization arguments such as the following claim applied to $\ell_M$. We omit the details. 

\begin{fact}[Consequence of Theorem 10.1 in \citet{DBLP:books/daglib/0034861}] \label{fact:standard-generalization-bound}
Let $H$ be a class of functions over $\R^d$ such that each $h \in H$ can be described in $B$ bits. Suppose each function in $H$ takes values in $[-M,M]$ for some positive real $M$.  
Let $\cD$ be a distribution on $\R^d \times [-M,M]$ and let $(x_1,y_1),\ldots,(x_n,y_n)$ be $n$ i.i.d samples from $\cD$ for $n > n_0 = O(M^2 B\log{(1/\delta)}/\epsilon^2)$. 

Then, with probability at least $1-\delta$ over the draw of $X$, for every $\ell \in H$, 
\[
\E_{(x,y) \sim D} (y - h(x))^2 \leq (1/n) \sum_{i=1}^n (y_i - h(x_i))^2 + \epsilon \mper
\]

\end{fact}

\ignore{
 The certifiability lemma above immediately implies why the output of the algorithm is a good regression hypothesis over the uncorrupted sample $X$. 

\begin{lemma}[Analysis of Algorithm] \label{lem:analysis-full-L2}

Let $X$ be a collection of $n$ labeled examples in $\R^d$ such that $\hat{D}$, the uniform distribution on  $x_1, x_2, \ldots, x_n$ is $(C,k)$-certifiably hypercontractive. 
Let $U$ be an $\eta$-corruption of $X$.
Let $\tmu$ be the pseudo-distribution satisfying $\cA_{U,\eta}$ that minimizes the polynomial $\Paren{\frac{1}{n} \sum_{i = 1}^n (y_i' - \iprod{\ell,x_i'})^2 }^{k/2}$. 

Then, for any $\ell^{*} \in \R^d$ of bit complexity at most $B = \poly(n,d^k)$, $\hat{\ell} = \pE_{\tmu}\ell$ satisfies: 

\[
\frac{1}{n} \sum_{i = 1}^n (y_i - \iprod{\hat{\ell},x})^2 < (1+O(C)\eta^{1-2/k})  \opt(\hat{\cD}) + O(C) \eta^{1-2/k} \Paren{\frac{1}{n} \sum_{i = 1}^n (y_i - \iprod{\ell^{*},x})^k}^{2/k}  \mper
\]

\end{lemma}

Next, we show that $\opt(\hat{\cD}) \leq \opt(\cD)$. Here, we employ the hypercontractivity of $D$ and the boundedness of the labels coupled with standard concentration arguments. 

\begin{lemma}[Convergence of Empirical Optimum] \label{lem:error-of-output-is-small}
Let $\cD$ be a distribution on $\R^d \times \cY$ such that the marginal $D = \cD_x$ is $(C,k)$-hypercontractive for for some even $k \geq 4$ and $\cY \subseteq [-M,M]$ for some positive real $M$. 
Let $\opt(\cD) = \min_{\ell} \E_{\cD} (y - \iprod{\ell,x})^2$ where the minimum is over all $\ell$ of bit complexity at most $B$. 
Let $X$ be a sample from $\cD$ of size $n \geq n_0$ for $n_0 = O(1/\delta \epsilon^2) \cdot M^4$ and let $\hat{\cD}$ be the uniform distribution on $X$. 
Let $\opt(\hat{\cD}) = \min_{\ell} \E_{\hat{\cD}} (y - \iprod{\ell,x})^2$ where the minimum is over all $\ell$ of bit complexity at most $B$.

Then, with probability at least $1-\delta$ over the draw of $X$,  $\opt(\hat{\cD}) \leq \opt(\cD) + \epsilon$.  

\end{lemma}

Using the above lemmas, it is easy to wrap up the proof of Theorem \ref{thm:analysis-L2-linear-regression}. We need to show three claims 1) if $D = \cD_x$ is $(C,k)$-certifiably hypercontractive, then with high probability over the draw of a large enough sample $X$,  so is $\hat{D}$, 2) when the original sample $X$ is of size large enough, then any regression hypothesis that has low error on uniform distribution $\hat{\cD}$ on $X$ has low error on $\cD$ and 3) for the optimal $\ell^{*}$, $\E_{\hat{\cD}}[ (y - \iprod{\ell^{*},x})^k]$ is close to that of $\E_{\cD} (y - \iprod{\ell^{*},x})^k$.

The second follows from standard generalization for bounded loss functions. Namely, we will appeal to the following standard generalization result and apply it to the truncation $f = f_{\ell}$ of any linear function $\ell$ defined by:
\[
f (x) = \begin{cases} \iprod{\hat{\ell},x} & \text{ if } |\iprod{\hat{\ell},x}| \leq M \\
       \sign(\iprod{\hat{\ell},x}) \cdot M & \text{ otherwise.}
       \end{cases}
\]

}
\ignore{

  \begin{proof}[Proof of Lemma \ref{lem:analysis-full-L2}]
Since $\tmu$ is a pseudo-distribution of level $k$, combining Fact \ref{fact:sos-soundness} and Lemma \ref{lem:identifiability-least-squares-linear-sos}, we must have for $C = C(k/2)$,
\begin{equation} \label{eq:proof-to-pseudo-dist}
\pE_{\tmu} \Paren{\err_{\hat{\cD}}(\ell) - \err_{\cD'}(\ell)}^{k/2} \leq O(C^{k/2} \eta^{k/2-1}) \cdot \pE_{\tmu} \err_{\cD'}(\ell)^{k/2} + O(C^{k/2}\eta^{k/2-1}) \cdot \Paren{\E_{\hat{\cD}} (y-\iprod{\ell^{*},x})^k} 
\end{equation}

Since $\tmu$ is of degree $k$ and $\err_{\cD'}$ is a SoS polynomial, $\pE_{\tmu} \err_{\cD'}(\ell) \geq 0.$ Thus, let $\opt(\hat{\cD})$ be a positive real such that: $\opt(\hat{\cD})^k = \pE_{\tmu} \err_{\cD'}(\ell)^k$.

Then, taking $2/k$th powers of both sides of \eqref{eq:proof-to-pseudo-dist}, we thus have:
\[
(\pE_{\tmu} \err_{\hat{\cD}}(\ell)- \err_{\cD'}(\ell))^{k/2})^{2/k} \leq O(C) \eta^{1-2/k} \opt(\hat{\cD}) + O(C) \eta^{1-2/k} \Paren{\E_\cD (y-\iprod{\ell^{*},x})^k}^{2/k}\mper
\]

Using H\"older's inequality for pseudo-distributions, we have:
$(\pE_{\tmu} \err_{\hat{\cD}} (\ell) - \err_{\cD'}(\ell))^{k/2} \leq \pE_{\tmu} (\err_{\hat{\cD}}(\ell) - \err_{\cD'}(\ell))^{k/2}$ and thus, 

\[
|\pE_{\tmu} \err_{\hat{\cD}} (\ell) - \err_{D'}(\ell)| \leq O(C) \eta^{1-2/k} \opt + O(C) \eta^{1-2/k} \Paren{\E_{\hat{\cD}} (y-\iprod{\ell^{*},x})^k}^{2/k}\mper
\]

By another application of H\"older's inequality for pseudo-distributions, we have: $\Paren{\pE_{\tmu} \err_{\hat{\cD}}(\ell)}^{k/2} \leq \pE_{\tmu} \err_{\hat{\cD}} (\ell)^{k/2}$. Thus, we have:
\[
\pE_{\tmu} \err_{\hat{\cD}} (\ell)  \leq (1+O(C \eta^{1-2/k})) \opt(\hat{\cD}) + O(C) \eta^{1-2/k} \Paren{\E_{\hat{\cD}} (y-\iprod{\ell^{*},x})^k}^{2/k}\mper
\]

Using H\"older's inequality for pseudo-distributions again, we have: 
$\err_{\hat{\cD}}(\pE_{\tmu}(\ell)) \leq \pE_{\tmu} \err_{\hat{\cD}} (\ell) .$ 

Plugging this in the above completes the proof.


\end{proof}

\begin{proof}[Proof of Lemma \ref{lem:error-of-output-is-small}]
Let $\ell^{*}$ be a linear function of bit complexity at most $B$ that achieves the optimum least squares regression error on $\cD$. Then, $\E_{\cD} \iprod{\ell^{*},x}^2 \leq 2\E_{\cD} (y - \iprod{\ell^{*},x})^2 + 2 \E_{\cD} y^2 \leq 2 \opt(\cD) + 2M^2$. Further, $\E_{\cD} (y - \iprod{\ell^{*},x})^4 \leq 8 \E_{\cD} y^4 + 8 \E_{\cD} \iprod{\ell^{*},x}^4$. By $(C,4)$-hypercontractivity of $\cD$,  $\E_{\cD} \iprod{\ell^{*},x}^4 \leq \Paren{\E_{\cD} \iprod{\ell^{*},x}^2}^2 \leq 4 \opt(\cD)^2 + 4 M^4$. Since $\opt(\cD) \leq M^2$ (the $0$ function achieves this error), this is at most $O(M^4)$.  

Next, let $X$ be an i.i.d. sample from $\cD$ and let $\hat{\cD}$ be the uniform distribution on $X$. We will show that the error of $\ell^{*}$ on $\hat{\cD}$ is at most an $\epsilon$ larger than the error of $\ell^{*}$ on $\cD$ with probability at least $1-\delta$. This will establish the claimed upper bound on $\opt(\hat{\cD})$.

Observe that the error of $\ell^{*}$ on $\hat{\cD}$: $\E_{\hat{\cD}} (y - \iprod{\ell^{*},x})^2$ is the average of independent draws of the random variable $Z= (y - \iprod{\ell^{*},x})^2$. 

From the above calculation, $\E Z^2 \leq O(1) \cdot (\opt(\cD)^4 + M^4)$. Thus, for some $n_0 = O(1/\delta \epsilon^2) (M^4)$, if we take $n \geq n_0$ independent samples $Z_1, Z_2, \ldots,Z_n$ of $Z$, then $\Pr[ |\frac{1}{n} \sum_{i = 1}^n Z_i - \E[Z]| \geq \epsilon] \leq \delta$. 

This finishes the proof.

\end{proof}
}
\subsection{Robust L1 Regression} \label{sec:robust-L1-regression-algo}
In this section, we present our robust L1 regression algorithm. Our main goal is the following theorem. 

\begin{theorem} \label{thm:L1-regression-analysis}
Let $\cD$ be an arbitrary distribution on $\R^d \times \cY$ for $\cY \subseteq [-M,M]$ for a positive real $M$. Let $\kappa$ be the ratio of the maximum to the minimum eigenvalue of the covariance matrix of $D$, the marginal of $\cD$ on $x$. Let $\opt(\cD)$ be the minimum of $\E_{\cD} \abs{y-\iprod{\ell,x}}$ over all $\ell$ that have bit complexity bounded above by $B$. Let $\ell^{*}$ be any such minimizer and $\eta > 0$ be an upper bound on the fraction of corruptions. 

For any $\epsilon > 0$, let $X$ be an i.i.d. sample from $\cD$ of size $n \geq n_0$ for some $n_0 = O(1/\epsilon^2) \cdot (M^2\|\ell^{*}\|_2^4 + d \log{(d)} \|\Sigma\|/\eta)$.

Then, with probability at least $1-\epsilon$ over the draw of the sample $X$, given any $\eta$-corruption $U$ of $X$ and $\eta$ as input, there's a polynomial time algorithm (Algorithm \ref{alg:robust-regression-program-L1}) that outputs a function $f:\R^d \times \R$ such that:
\[
\E_{(x,y) \sim \cD} \abs{y- f(x)} < \opt(\cD) + O(\sqrt{\kappa \eta}) (\sqrt{\E_{\cD} y^2} + \sqrt{\Paren{\E_{\cD} (y-\iprod{\ell^{*},x})^2}}) + \epsilon \mper
\] 

\end{theorem}
\begin{remark}
The lower bound example in Lemma \ref{lm:lb1} also shows that the above bound is tight in the dependence on $\eta$ and $\kappa$.
\end{remark}

As in the previous section, our algorithm will find pseudo-distributions satisfying a set of polynomial inequalities that encode the hypotheses of the robust certifiability lemma and the ``error'' polynomial. 

Let $\cA_{U,\eta,Q}$ be the following system of polynomial equations:

\begin{equation}
  \cA_{U,\eta,Q}\colon
  \left \{
    \begin{aligned}
      &&
      \textstyle\sum_{i=1}^n w_i
      &= (1-\eta) \cdot n\\
      &\forall i\in [n].
      & w_i^2
      & =w_i \\
      &\forall i\in [n].
      & w_i \cdot (u_i - x'_i)
      & = 0\\
      &\forall i\in [n].
      & w_i \cdot (v_i - y'_i)
      & = 0\\
      & &\|\ell\|_2^2 &\leq Q^2\\
      &\forall i \in [n] & \tau_i' \geq (y_i' - \iprod{\ell,x'_i})\\
      &\forall i \in [n] & \tau_i' \geq -(y_i' - \iprod{\ell,x'_i})
    \end{aligned}
  \right \} \label{eq:polynomial-constraints-w}
\end{equation}

This system of equations takes as parameters the input sample $U$ and a bound on the fraction of outliers $\eta$. 

We can now describe our algorithm for robust L1 regression. 

\begin{mdframed}
  \begin{algorithms}[Algorithm for Robust L1 Linear Regression via Sum-of-Squares]
     \label{alg:robust-regression-program-L1}\mbox{}
    \begin{description}
    \item[Given:]
     An $\eta$-corruption $U$ of a labeled sample $X$ of size $n$ from an arbitrary distribution $\cD$. 
    The Euclidean norm of the best fitting L1 regression hypothesis for $\cD$, $Q$.
    \item[Operation:]\mbox{}
      \begin{enumerate}
      \item 
        find a level-$4$ pseudo-distribution $\tmu$ that satisfies $\cA_{U,\eta,Q}$ and minimizes $\Paren{\frac{1}{n}\sum_{i = 1}^n \tau_i}^2$.
      \item
        Return $\hat{\ell} = \pE_{\tmu}\ell$.
      \end{enumerate}
    \end{description}    
  \end{algorithms}
\end{mdframed}

\paragraph{Analysis of Algorithm}
The plan of this subsection and the proofs are essentially analogous to the ones presented in the previous subsection. We will split the analysis into bounding the optimization and generalization errors as before. Let $\opt_{SoS}$ be the L1 error of $\hat{\ell}$ output by Algorithm \ref{alg:robust-regression-program-L1} and let $\ell^{*}$ be the optimal hypothesis for $\cD$.

\begin{lemma}[Bounding the Optimization Error] \label{lem:opt-error-L1}
Under the assumptions of Theorem \ref{thm:L1-regression-analysis} (and following the above notations), 
$$\err_{\widehat{\cD}}(\widehat{\ell}) \leq \widehat{\opt}_{SOS} + 2 \kappa^{1/2} \eta^{1/2} \sqrt{\sum_{i =1}^n y_i^2} + 2 \kappa^{1/2} \eta^{1/2} \err_{\cD}(\ell^{*})\mper$$ 


\end{lemma}

\begin{lemma}[Bounding the Generalization Error]
Under the assumptions of Theorem \ref{thm:L1-regression-analysis}, with probability at least $1-\epsilon$, 
\begin{enumerate}
\item $\widehat{\opt}_{SOS}  \leq \opt(\cD) + \epsilon$.

\item $\err_{\cD}(\widehat{\ell}_M)) \leq \err_{\widehat{\cD}}(\widehat{\ell}) + \epsilon$.

\item $\frac{1}{n} \sum_{i = 1}^n y_i^2 \leq \E_{\cD} y^2$.
\end{enumerate}
\end{lemma}

The proofs of the above two lemmas are entirely analogous to the ones presented in the previous section. The main technical ingredient as before is a SoS version of the robust certifiability result. Since this is the only technical novelty in this subsection, we present the statement and proof of this result below and omit the other proofs.


\begin{lemma}[SoS Proof of Robust Certifiability for L1 Regression] \label{lem:robust-certifiability-L1-SoS}

Let $X$ be a collection of $n$ labeled examples in $\R^d$ such that $D$, the uniform distribution on  $x_1, x_2, \ldots, x_n$ has 2nd moment matrix with all eigenvalues within a factor $\kappa$ of each other. Let $U$ be an $\eta$-corruption of $X$.

Let $w,\ell,X', \tau'$ satisfy the set of system of polynomial equations $\cA_{U,\eta,Q}$. Let $\tau_i$ satisfy $\tau_i^2 = (y_i - \iprod{\ell,x_i})^2$ and $\tau_i \geq 0$ for every $i$. Then, for any $\ell^{*} \in \R^d$ such that $\|\ell^{*}\|_2 \leq Q$,



\[
\cA_{U,\eta,Q} \sststile{4}{w,\tau',X'} \frac{1}{n} \sum_{i = 1}^n \tau_i \leq \frac{1}{n} \sum_{i = 1}^n \tau_i' + 2\kappa^{1/2} \eta^{1/2} \sqrt{\frac{1}{n} \sum_{i = 1}^n y_i^2} + 2 \kappa^{1/2} \eta^{1/2} \cdot \sqrt{ \frac{1}{n} \sum_{i = 1}^n (y_i - \iprod{\ell^{*},x_i})^2}\mper
\]





\end{lemma}

\begin{proof}[Proof of Lemma \ref{lem:robust-certifiability-L1-SoS}]

For every $i \in [n]$, define $w'_i = w_i$ iff $(x_i,y_i)$ is uncorrupted in $U$.
Then, observe that $\sum_i w_i' = s$ for $s \geq (1-2\epsilon)n$ and that $\sststile{w}{2} \Set{w_i^2 - w_i = 0}$. 

Then, 
\[
\sststile{w'}{2} \Set{ \frac{1}{n} \sum_i (1-w'_i)^2 \leq 2 \epsilon}\mper
\]

Thus, we have:
\[
\sststile{w,\ell, \tau'}{4} \frac{1}{n} \sum_{i = 1}^n \tau_i = \frac{1}{n} \sum_{i = 1}^n w'_i \tau_i + \sum_{i = 1}^n (1-w'_i) \cdot \tau_i
\]

Further, it's easy to verify by direct expansion that:
\[
\Set{w_i(x_i -x_i') = 0, w_i (y_i - y_i') = 0 \mid \forall i} \sststile{w'}{4} \Set{ w'_i (\tau_i -\tau'_i) = 0 \mid \forall i}
\]

As a result, we have:
\[
\sststile{w,\ell, \tau'}{4} \frac{1}{n} \sum_{i = 1}^n \tau_i = \frac{1}{n} \sum_{i = 1}^n w'_i \tau_i' + \sum_{i = 1}^n (1-w'_i) \cdot \tau_i
\]

For brevity, let's write $\err_{D}(\ell) = \sum_{i =1}^n \sum_{i = 1}^n \tau_i$ and $\err_{D'} (\ell) = \sum_{i =1}^n \sum_{i = 1}^n \tau_i'$. Then, we have:

Using the sum-of-squares vesion of the Cauchy-Shwarz inequality, we have:

\begin{align}
\sststile{w,\ell, \tau'}{4} \Paren{\err_D(\ell) - \err_{D'}(\ell)}^{2} &= \Paren{\frac{1}{n} \sum_{i = 1}^n (1-w'_i) \cdot \tau_i}^2\notag\\
&\leq \Paren{\frac{1}{n} \sum_{i = 1}^n (1-w'_i)}^2 \frac{1}{n} \sum_{i =1}^n (y_i - \iprod{\ell,x_i})^2 \notag\\
&\leq \epsilon \frac{1}{n} \sum_{i =1}^n (y_i - \iprod{\ell,x_i})^2  \label{eq:CS-bound-L1} \mper
\end{align}

Next, we have:
\[
\sststile{\ell}{4} \Set{ \frac{1}{n} \sum_{i =1}^n (y_i - \iprod{\ell,x_i})^2 \leq 2\frac{1}{n} \sum_{i =1}^n (y_i - \iprod{\ell^{*},x_i})^2 + 2\frac{1}{n} \sum_{i =1}^n \iprod{\ell-\ell^{*},x_i})^2} \mper
\]

Further, we also have:
\[
\sststile{\ell}{4} \Set{ \frac{1}{n} \sum_{i =1}^n \iprod{\ell-\ell^{*},x_i})^2 \leq 2\frac{1}{n} \sum_{i =1}^n \iprod{\ell,x_i})^2 + 2 \frac{1}{n} \sum_{i =1}^n \iprod{\ell^{*},x_i})^2 }
\]

Using that for any PSD matrix $A$, we have the SoS inequality $\|x\|_2^2 \|A\|_{min} \leq x^{\top} A x \leq \|x\|_2^2 \|A\|_{max}$ where $\|A\|_{max}$ and $\|A\|_{min}$ are the largest and smallest singular values of $A$, respectively, we have:

\[
\sststile{\ell}{4} \Set{ \frac{1}{n} \sum_{i =1}^n \iprod{\ell,x_i})^2 \leq \kappa \frac{1}{n} \sum_{i =1}^n \iprod{\ell^{*},x_i})^2  }
\]

Finally, we also have:
\[
\sststile{\ell}{4} \Set{\frac{1}{n} \sum_{i =1}^n \iprod{\ell^{*},x_i})^2 \leq \frac{1}{n} \sum_{i =1}^n (y_i - \iprod{\ell^{*},x_i})^2  + 2\frac{1}{n} \sum_{i =1}^n y_i^2 }
\]

Combining the above inequalities with \eqref{eq:CS-bound-L1} yields the lemma.

\end{proof}

%% file: content/lower-bound.tex

\section{Statistical Limits of Outlier-Robust Regression}
\label{sec:lower-bounds}
Here we exhibit statistical lower bounds for what can be achieved for outlier-robust regression. In particular, these simple examples illustrate strong separations between regression and regression in the presence of contamination and also demonstrate the necessity of our disributional assumptions. 


\paragraph{Necessity of Distributional Assumptions}
A classical result in analysis of regression is \emph{consistency} of the least-squares estimator when the labels are bounded. Concretely, let $\cD$ be a distribution on $\R^d \times [-1,1]$. Let $(x_1,y_1),\ldots,(x_n,y_n)$ be i.i.d samples from $\cD$. Let 
$\hat{\ell} = \arg\min_{\ell} \sum_{i=1}^n (y_i - \iprod{\ell,x_i})^2,$
be the least-squares estimator. Then, (say, via Theorem 11.3 in \cite{DBLP:books/daglib/0035701}) 
$\err_{\cD}(\hat{\ell}) \leq \frac{O(d)}{n} + 8 \cdot \arg\min_{\ell}\err_{\cD}(\ell).$

In particular, in the realizable-case, i.e., when $(x,y) \sim \cD$ satisfies $y = \iprod{\ell^*,x}$, the error of the least-squares estimator approaches zero as $n \rightarrow \infty$ irrespective of the marginal distribution $\cD_X$. 

Given the above bound, it is natural to ask if we could get a similar marginal-distribution independent bound in the presence of outliers: Does there exist an estimator which achieves error $h(\eta)$ with $\eta$-fraction of corruptions for some function $h:\R \to \R$ with $h \rightarrow 0$ as $\eta \rightarrow 0$? It turns out that this is statistically impossible in the presence of sample contamination.
\begin{lemma}\label{lm:lb1}
There is a universal constant $c > 0$ such that the following holds. For all $\eta > 0$, there is no algorithm that given $\eta$-corrupted samples\footnote{The lemma also holds in the weaker \emph{Huber's} contamination model even though we do not study this model in this work.} $(x,y)$ from distributions $\tilde{\cD}$ on $\R^d \times [-1,1]$ finds a hypothesis vector $\ell \in \R^d$ such that $\E[\err_{\tilde{\cD}}(\ell)] < c$. 
\end{lemma}

\ignore{
\begin{lemma}
There exists a universal constant $c > 0$ such that for all $\eta > 0$, there exist two distributions $\cD, \cD'$ on $\R^2 \times [-1,1]$ and associated random variables $(x,y), (x',y')$ on $\R^2 \times \R$ satisfying:
\begin{enumerate}
\item Realizability: There exist $\ell, \ell' \in \R^2$ such that $y = \iprod{\ell,x}$ and $y' = \iprod{\ell',x}$ for some $\ell, \ell'$. 
\item Closeness in statistical distance: $\|\cD - \cD'\|_{TV} \leq \eta$. 
\item For any $w \in \R^2$, $\max(\err_\cD(w), \err_{\cD'}(w)) \geq c$. 
\end{enumerate}
\end{lemma}
\begin{proof}
Fix $\kappa \geq 2$ and let $D$ be the distribution of the random variable sampled as follows: 1) Sample $\alpha$ uniformly at random from $[-1,1]$; 2) With probability $1-\eta$ output $(\alpha, \alpha)$; 3) With probability $\eta$ output $(\kappa\cdot \alpha,\alpha)$.

Let $D'$ be the distribution of the random variable sampled as follows: 1) Sample $\alpha$ uniformly at random from $[-1,1]$; 2) With probability $1-\eta$ output $(\alpha, \alpha)$; 3) With probability $\eta$ output $(\alpha,\kappa \alpha)$.

Let $\ell = (0,1)$, $\ell' = (1,0)$. Finally, let $\cD$ be the distribution of $(x, \iprod{\ell,x})$ for $x \sim D$ and let $\cD'$ be the distribution of $(x', \iprod{\ell',x'})$ for $x' \sim D'$. 

It is easy to check that $\|\cD - \cD'\|_{TV} \leq \eta$. Finally, it is not too hard to check that for any $w$, 
$$\err_{\cD}(w) + \err_{\cD'}(w) \geq \Omega(1) \cdot \frac{ \eta \kappa^2}{1 + \eta \kappa^2}.$$

It follows that for some universal constant $c > 0$, 
$$\min(\err_\cD(w), \err_{\cD'}(w)) \geq c \min(1, \eta \kappa^2).$$
The claim now follows. 
\end{proof}}

\begin{proof}
Suppose there is an algorithm as above. Let $\delta = \eta/(2-\eta)$ and let $\kappa \geq 2$ be sufficiently large to be chosen later. 
Let $\cD$ be the distribution of the random variable on $\R^2 \times \R$ samples as follows: 1) Sample $\alpha$ uniformly at random from $[-1,1]$; 2) With probability $1-\eta$ output $((\alpha, \alpha), \alpha)$; 3) With probability $\eta$ output $((\kappa\cdot \alpha,\alpha),\alpha)$. Note that for $(x,y) \sim \cD$, $y = \iprod{\ell,x}$ for $\ell = (0,1)$. 

Similarly, let $\cD'$ be the distribution of the random variable on $\R^2 \times \R$ samples as follows: 1) Sample $\alpha$ uniformly at random from $[-1,1]$; 2) With probability $1-\eta$ output $((\alpha, \alpha), \alpha)$; 3) With probability $\eta$ output $((\alpha,\kappa \cdot \alpha),\alpha)$. Note that for $(x',y') \sim \cD$, $y' = \iprod{\ell',x'}$ for $\ell = (1,0)$. 

It follows from a few elementary calculations that for any $w \in \R^2$, 
$\err_{\cD}(w) + \err_{\cD'}(w) \geq \Omega(1) \cdot \frac{ \eta \kappa^2}{1 + \eta \kappa^2}.$

It follows that for some universal constant $c > 0$, and $\kappa = 1/\sqrt{\eta}$, $\min(\err_\cD(w), \err_{\cD'}(w)) \geq c .$

Finally, let $\cD''$ be the distribution of the random variable sampled as follows: 
1) Sample $\alpha$ uniformly at random from $[-1,1]$; 2) With probability $1-\delta$ output $((\alpha, \alpha), \alpha)$; 3) With probability $\delta/2$ output $((\kappa\cdot \alpha,\alpha),\alpha)$; 4) With probability $\delta/2$ output $((\alpha, \kappa \cdot \alpha),\alpha)$. 

Note that $\cD''$ can be obtained by a $(\eta/2)$-corruption of $\cD$ as well as $\cD'$. On the other hand, for any $w \in \R^2$, one of $\err_\cD(w), \err_{\cD'}(w)$ is at least $c$ where $c$ is the constant from the previous lemma. Thus no algorithm can output a good hypothesis for both $\cD$ or $\cD'$. The claim now follows. 
\end{proof}

%% file: main.bbl
\newcommand{\etalchar}[1]{$^{#1}$}
\def\dbar{\leavevmode\hbox to 0pt{\hskip.2ex \accent"16\hss}d}
  \def\dbar{\leavevmode\hbox to 0pt{\hskip.2ex \accent"16\hss}d}
  \def\dbar{\leavevmode\hbox to 0pt{\hskip.2ex \accent"16\hss}d}
  \def\romsup#1{{\edef\next{\the\font}$^{\next#1}$}}
  \def\polhk#1{\setbox0=\hbox{#1}{\ooalign{\hidewidth
  \lower1.5ex\hbox{`}\hidewidth\crcr\unhbox0}}}
  \def\polhk#1{\setbox0=\hbox{#1}{\ooalign{\hidewidth
  \lower1.5ex\hbox{`}\hidewidth\crcr\unhbox0}}}
  \def\polhk#1{\setbox0=\hbox{#1}{\ooalign{\hidewidth
  \lower1.5ex\hbox{`}\hidewidth\crcr\unhbox0}}}
  \def\polhk#1{\setbox0=\hbox{#1}{\ooalign{\hidewidth
  \lower1.5ex\hbox{`}\hidewidth\crcr\unhbox0}}}
  \def\polhk#1{\setbox0=\hbox{#1}{\ooalign{\hidewidth
  \lower1.5ex\hbox{`}\hidewidth\crcr\unhbox0}}}
  \def\ocirc#1{\ifmmode\setbox0=\hbox{$#1$}\dimen0=\ht0 \advance\dimen0
  by1pt\rlap{\hbox to\wd0{\hss\raise\dimen0
  \hbox{\hskip.2em$\scriptscriptstyle\circ$}\hss}}#1\else {\accent"17 #1}\fi}
  \def\cfac#1{\ifmmode\setbox7\hbox{$\accent"5E#1$}\else
  \setbox7\hbox{\accent"5E#1}\penalty 10000\relax\fi\raise 1\ht7
  \hbox{\lower1.15ex\hbox to 1\wd7{\hss\accent"13\hss}}\penalty 10000
  \hskip-1\wd7\penalty 10000\box7}
  \def\cfac#1{\ifmmode\setbox7\hbox{$\accent"5E#1$}\else
  \setbox7\hbox{\accent"5E#1}\penalty 10000\relax\fi\raise 1\ht7
  \hbox{\lower1.15ex\hbox to 1\wd7{\hss\accent"13\hss}}\penalty 10000
  \hskip-1\wd7\penalty 10000\box7}
  \def\cfac#1{\ifmmode\setbox7\hbox{$\accent"5E#1$}\else
  \setbox7\hbox{\accent"5E#1}\penalty 10000\relax\fi\raise 1\ht7
  \hbox{\lower1.15ex\hbox to 1\wd7{\hss\accent"13\hss}}\penalty 10000
  \hskip-1\wd7\penalty 10000\box7}
  \def\ocirc#1{\ifmmode\setbox0=\hbox{$#1$}\dimen0=\ht0 \advance\dimen0
  by1pt\rlap{\hbox to\wd0{\hss\raise\dimen0
  \hbox{\hskip.2em$\scriptscriptstyle\circ$}\hss}}#1\else {\accent"17 #1}\fi}
  \def\ocirc#1{\ifmmode\setbox0=\hbox{$#1$}\dimen0=\ht0 \advance\dimen0
  by1pt\rlap{\hbox to\wd0{\hss\raise\dimen0
  \hbox{\hskip.2em$\scriptscriptstyle\circ$}\hss}}#1\else {\accent"17 #1}\fi}
  \def\polhk#1{\setbox0=\hbox{#1}{\ooalign{\hidewidth
  \lower1.5ex\hbox{`}\hidewidth\crcr\unhbox0}}}
  \def\ocirc#1{\ifmmode\setbox0=\hbox{$#1$}\dimen0=\ht0 \advance\dimen0
  by1pt\rlap{\hbox to\wd0{\hss\raise\dimen0
  \hbox{\hskip.2em$\scriptscriptstyle\circ$}\hss}}#1\else {\accent"17 #1}\fi}
  \def\polhk#1{\setbox0=\hbox{#1}{\ooalign{\hidewidth
  \lower1.5ex\hbox{`}\hidewidth\crcr\unhbox0}}}
  \def\cfudot#1{\ifmmode\setbox7\hbox{$\accent"5E#1$}\else
  \setbox7\hbox{\accent"5E#1}\penalty 10000\relax\fi\raise 1\ht7
  \hbox{\raise.1ex\hbox to 1\wd7{\hss.\hss}}\penalty 10000 \hskip-1\wd7\penalty
  10000\box7} \def\cfac#1{\ifmmode\setbox7\hbox{$\accent"5E#1$}\else
  \setbox7\hbox{\accent"5E#1}\penalty 10000\relax\fi\raise 1\ht7
  \hbox{\lower1.15ex\hbox to 1\wd7{\hss\accent"13\hss}}\penalty 10000
  \hskip-1\wd7\penalty 10000\box7} \def\cdprime{$''$}
  \def\polhk#1{\setbox0=\hbox{#1}{\ooalign{\hidewidth
  \lower1.5ex\hbox{`}\hidewidth\crcr\unhbox0}}}
  \def\cftil#1{\ifmmode\setbox7\hbox{$\accent"5E#1$}\else
  \setbox7\hbox{\accent"5E#1}\penalty 10000\relax\fi\raise 1\ht7
  \hbox{\lower1.15ex\hbox to 1\wd7{\hss\accent"7E\hss}}\penalty 10000
  \hskip-1\wd7\penalty 10000\box7}
  \def\ocirc#1{\ifmmode\setbox0=\hbox{$#1$}\dimen0=\ht0 \advance\dimen0
  by1pt\rlap{\hbox to\wd0{\hss\raise\dimen0
  \hbox{\hskip.2em$\scriptscriptstyle\circ$}\hss}}#1\else {\accent"17 #1}\fi}
\providecommand{\bysame}{\leavevmode\hbox to3em{\hrulefill}\thinspace}
\providecommand{\MR}{\relax\ifhmode\unskip\space\fi MR }
\providecommand{\MRhref}[2]{%
  \href{http://www.ams.org/mathscinet-getitem?mr=#1}{#2}
}
\providecommand{\href}[2]{#2}

%% file: content/DeferredStatements.tex
\section{Outlier-Robust Polynomial Regression}\label{sec:app-moved}

Our arguments also extend straightforwardly to get similar guarantees for polynomial regression. We elaborate on these next.

The following extends the definition of hypercontractivity to polynomials.

\begin{definition}[Certifiable polynomial hypercontractivity]\label{def:hyperconc2}
For a function $C:[k] \to \R_+$, we say a distribution $D$ on $\R^d$ is $k$-cerifiably $(C,t)$-hypercontractive if for every $r$ such that $2r t \leq k$, there is a degree $k$ sum of squares proof of the following inequality in variable $P$ where $p$ stands for $\langle P,x^{\otimes t}\rangle.$
\[
\E_D p(x)^{2r} \leq \Paren{(C(r) \E_{D} p(x)^{2}}^{r}.
\] 
\end{definition}
Many natural distributions satisfy certifiably hypercontractivity \citep{DBLP:conf/soda/KauersOTZ14} for polynomials such as gaussian distributions and the product distributions on the hypercube $\zo^n$ with all coordinate marginals in $(0,1)$. Our results will apply to all such distributions. 

Next, we state an extension of our robust certification lemma for polynomial regression. The proof is essentially the same as that of Lemma \ref{lem:identifiability-least-squares-linear}. 
\begin{lemma}[Robust Generalization for polynomial regression]
Fix $k,t \in \N$ and let $\cD,\cD'$ be distributions on $\R^d \times \R$ such that $\|\cD-\cD'\|_{TV} \leq \epsilon$ and the marginal $\cD_X$ of $\cD$ on $x$ is $k$-certifiably $(C,t)$-hypercontractive for some $C:[k] \to \R_+$ and for some even integer $k \geq 4$.

Then, for any degree at most $t$ polynomials $p,p^*:\R^d \to \R$, and any $\epsilon$ such that $2 C(k/2) \epsilon^{1-2/k} < 0.9$, we have:
\[
\err_{\cD}(p) \leq (1+O(C(k/2)) \epsilon^{1-2/k}) \cdot \err_{\cD'}(p) + O(C(k/2))\epsilon^{1-2/k} \cdot \Paren{\E_\cD (y-p^*(x))^k}^{2/k}\mper\]\label{lem:identifiability-least-squares-polynomial}
\end{lemma}

\begin{theorem} \label{thm:polyregression}
Let $\cD$ be a distribution on $\R^d \times [-M,M]$ for some positive real $M$ such the marginal on $\R^d$ is $(C,k)$-certifiably hypercontractive distribution for degree $t$ polynomials. Let $\opt_B(\cD) = \min_{p} \E_{\cD}[ (y - p(x))^2]$ where the minimum is over all polynomials $p$ of degree $t$ and bit complexity $B$. Let $p^{*}$ be any such minimizer. 

Fix any even $k \geq 4$ and any $\epsilon > 0$. Let $X$ be an i.i.d. sample from $\cD$ of size $n \geq n_0 = \poly(d^k, B, M, 1/\epsilon)$. 
Then, with probability at least $1-\epsilon$ over the draw of the sample $X$, given any $\eta$-corruption $U$ of $X$ and $\eta$ as input, there is a polynomial time algorithm (Algorithm \ref{alg:robust-regression-program}) that outputs a $\ell \in \R^d$ such that for $C = C(k/2)$, 
\[
\err_{\cD}(p_M) < (1 + O(C)\eta^{1-2/k}) \opt_B(\cD) + O(C) \eta^{1-2/k} \Paren{\E_{\cD} (y-p^{*}(x))^k}^{2/k} + \epsilon
.
\] 
\end{theorem}